\documentclass[11pt,letterpaper]{mystyle}
\usepackage{fancyhdr}

\usepackage{tikz}
\usepackage[utf8]{inputenc} 
\usepackage[T1]{fontenc}
\usepackage[numbers]{natbib}
\usepackage{adjustbox}
\usepackage{breakurl}    
\usepackage{hyperref}
\usepackage[edges]{forest}
\usepackage{subcaption}
\usepackage{soul}
\usepackage{multirow}
\usepackage[utf8]{inputenc} 
\usepackage{booktabs}  
\usepackage{amsfonts}       
\usepackage{nicefrac}      
\usepackage{microtype}     
\usepackage{xcolor}        
\usepackage{colortbl}
\usepackage{enumitem}
\usepackage[numbers]{natbib}
\usepackage{amssymb}
\usepackage{wrapfig}
\usepackage{courier}
\usepackage{CJKutf8}
\usepackage{ragged2e}
\usepackage{longtable}
\usepackage[useregional]{datetime2}
\usepackage{threeparttable}
\definecolor{tamuMaroon}{HTML}{500000}
\colorlet{tamuGrayMaroon}{tamuMaroon!15!black}
\newcommand{\github}{\raisebox{-1.5pt}{\includegraphics[height=1.05em]{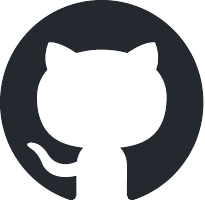}}}
\renewcommand{\today}{\DTMtoday}

\fancypagestyle{headstyle}{
    \fancyhead[C]{
        When Do Multi-Agent Systems Help? An Information Bottleneck Perspective
    }
}

\definecolor{hidden-red}{RGB}{205, 44, 36}
\definecolor{hidden-blue}{RGB}{194,232,247}
\definecolor{hidden-orange}{RGB}{243,202,120}
\definecolor{hidden-green}{RGB}{34,139,34}
\definecolor{hidden-pink}{RGB}{255,245,247}
\definecolor{hidden-black}{RGB}{20,68,106}
\definecolor{purple}{RGB}{144,153,196}
\definecolor{yellow}{RGB}{255,228,123}
\definecolor{hidden-yellow}{RGB}{255,248,203}
\definecolor{tkcolor}{RGB}{224,223,255}
\definecolor{darkblue}{rgb}{0, 0.40, 0.75}

\definecolor{tamuMaroon}{HTML}{500000}
\colorlet{abstractTextColor}{tamuMaroon!15!black}
\newcommand{\abstractstyle}{\color{abstractTextColor}}

\makeatletter
\let\oldabstract\abstract
\let\oldendabstract\endabstract
\renewcommand{\abstract}{\oldabstract\abstractstyle}
\renewcommand{\endabstract}{\oldendabstract}
\makeatother

\hypersetup{colorlinks=true, citecolor=darkblue, linkcolor=darkblue, urlcolor=darkblue}
\tcbset{
  aibox/.style={
    width=\linewidth,
    top=8pt,
    bottom=4pt,
    colback=blue!6!white,
    colframe=black,
    colbacktitle=black,
    enhanced,
    center,
    attach boxed title to top left={yshift=-0.1in,xshift=0.15in},
    boxed title style={boxrule=0pt,colframe=white,},
  }
}

\newtcolorbox{TakeawayBox}[2][]{takeawaybox,title=#2,#1}

\usepackage[utf8]{inputenc} % allow utf-8 input
\usepackage[T1]{fontenc}    % use 8-bit T1 fonts
\usepackage{hyperref}       % hyperlinks
\usepackage{url}            % simple URL typesetting
\usepackage{booktabs}       % professional-quality tables
\usepackage{amsfonts}       % blackboard math symbols
\usepackage{nicefrac}       % compact symbols for 1/2, etc.
\usepackage{microtype}      % microtypography
\usepackage{xcolor}         % colors
\usepackage{float}
\usepackage{amsmath}
\usepackage{amsmath,amssymb,amsthm}
\usepackage{enumitem}
\usepackage{float}
\usepackage{dblfloatfix}
\usepackage{booktabs}
\usepackage{graphicx} 
\usepackage{placeins}
\usepackage[table]{xcolor}
\usepackage{amsmath, amssymb}
\usepackage{booktabs}
\usepackage[table]{xcolor}
\usepackage{xcolor}
\usepackage{booktabs}
\definecolor{gainPosSoft}{RGB}{232,244,235}  % soft green
\definecolor{gainNegSoft}{RGB}{249,232,229}  % soft red

\definecolor{planLower}{RGB}{235,243,250}
\definecolor{planHigher}{RGB}{252,239,224}

\definecolor{relayGreen}{RGB}{82,158,113}
\definecolor{relayOrange}{RGB}{210,137,58}
\definecolor{relayRed}{RGB}{200,96,86}

\newcommand{\gainpos}[1]{\cellcolor{gainPosSoft}#1}
\newcommand{\gainneg}[1]{\cellcolor{gainNegSoft}#1}

\newcommand{\planhigher}[1]{\cellcolor{planHigher}#1}
\newcommand{\planlower}[1]{\cellcolor{planLower}#1}

\usepackage{array}
\usepackage{cleveref}  
\usepackage{bm}
\usepackage{mathtools}
\usepackage{wrapfig}

\definecolor{gainGreen}{HTML}{2E8B57}
\definecolor{lossRed}{HTML}{D65F5F}

\tcbuselibrary{most}  % most bundle covers skins, breakable, theorems, listings, etc.

\newtcbtheorem[number within=section]{propositionbox}{Proposition}%
{
  colback=gray!8,
  colframe=gray!45,
  fonttitle=\bfseries,
  coltitle=black,
  boxrule=0.5pt,
  arc=2pt,
  left=6pt,
  right=6pt,
  top=5pt,
  bottom=5pt
}{prop}

\newtcbtheorem[number within=section]{theorembox}{Theorem}%
{
  colback=gray!8,
  colframe=gray!45,
  fonttitle=\bfseries,
  coltitle=black,
  boxrule=0.5pt,
  arc=2pt,
  left=6pt,
  right=6pt,
  top=5pt,
  bottom=5pt
}{theorem}

\tcbset{
  ibinsight/.style={
    colback=black!2,   
    colframe=black!60,
    boxrule=0.4pt,
    arc=1mm,
    left=1.2mm, right=1.2mm, top=0.8mm, bottom=0.8mm,
    enhanced,
  }
}

\newtcbtheorem[number within=section]{definitionbox}{Definition}%
{
  colback=gray!8,
  colframe=gray!45,
  fonttitle=\bfseries,
  coltitle=black,
  boxrule=0.5pt,
  arc=2pt,
  left=6pt,
  right=6pt,
  top=5pt,
  bottom=5pt
}{def}

\usepackage{listings}
\usepackage{xcolor}

\definecolor{promptbg}{RGB}{248,248,248}
\definecolor{promptborder}{RGB}{185,185,185}

\usepackage{xcolor}
\usepackage{listings}
\tcbuselibrary{listings,breakable,skins,theorems}

\definecolor{promptbg}{RGB}{248,248,248}
\definecolor{promptborder}{RGB}{130,130,130}

\lstdefinestyle{promptstyle}{
  basicstyle=\ttfamily\small,
  breaklines=true,
  columns=fullflexible,
  keepspaces=true,
  showstringspaces=false
}

\lstdefinestyle{promptstyle}{
    basicstyle=\ttfamily\scriptsize,
    breaklines=true,
    breakatwhitespace=false,
    columns=fullflexible,
    keepspaces=true,
    showstringspaces=false,
    frame=none,
    xleftmargin=0pt,
    xrightmargin=0pt,
    literate={—}{{--}}2 {→}{{$\rightarrow$}}1 {–}{{--}}2 {“}{{``}}1 {”}{{''}}1 {‘}{{`}}1 {’}{{'}}1 {…}{{...}}3
}

\newtcblisting{promptbox}[1][]{
    enhanced,
    breakable,
    width=\linewidth,
    colback=promptbg,
    colframe=promptborder,
    boxrule=0.45pt,
    arc=1.5pt,
    left=6pt,
    right=6pt,
    top=5pt,
    bottom=5pt,
    boxsep=0pt,
    listing only,
    listing options={
        style=promptstyle,
        linewidth=\linewidth
    },
    title=#1,
    fonttitle=\bfseries\small,
    coltitle=black,
    attach boxed title to top left={xshift=6pt,yshift=-2pt},
    boxed title style={
        colback=white,
        colframe=promptborder,
        boxrule=0.4pt,
        arc=1pt,
        left=4pt,
        right=4pt,
        top=1pt,
        bottom=1pt
    }
}

\theoremstyle{definition}

\theoremstyle{plain}

\theoremstyle{remark}

\title{When Do Multi-Agent Systems Help? \\ An Information Bottleneck Perspective}

\author{
  Wendi Yu$^{1*}$, 
  Lianhao Zhou$^{1*}$, 
  Xiangjue Dong$^{1}$, 
  Sai Sudarshan Barath$^{1}$, 
  Declan Staunton$^{1}$, 
  Byung-Jun Yoon$^{2,3}$, 
  Xiaoning Qian$^{1,2,3}$, 
  James Caverlee$^{1}$, 
  Shuiwang Ji$^{1}$
\\

\vspace{1mm}

\normalfont{
$^1$ Department of Computer Science \& Engineering, Texas A\&M University\vspace{-5pt} \\
$^2$ Department of Electrical \& Computer Engineering, Texas A\&M University\vspace{-5pt} \\
$^3$ Computing and Data Sciences Directorate, Brookhaven National Laboratory\vspace{-5pt} \\
}}

\begin{document}

\begin{abstract}

  % \vspace{1mm}
  \textbf{\large Abstract:}
  \vspace{1mm}

LLM powered multi-agent systems (MAS) have emerged as a promising paradigm for complex tasks. However, their advantages over single-agent systems (SAS) remain unclear, with performance varying inconsistently across settings. Here, we provide an information bottleneck perspective on elucidating the differences between MAS and SAS. Specifically, our key observation is that a SAS accumulates its full reasoning trace in one shared context, while a MAS uses isolated local contexts connected by bounded relay messages.
We show that, under infinite relay bandwidth, any SAS can be simulated by a MAS that transmits the full upstream context. Thus, the nontrivial advantage of MAS arises under bounded relays, where compression introduces a fundamental trade-off: reducing redundant context can improve efficiency, but may also incur loss of task-relevant information. We formalize this trade-off as an information bottleneck controlled by an effective parameter $\beta$, which captures how the balance shifts with model capability, and shows that MAS gains arise when context reduction outweighs relay information loss.
We conduct 18 controlled experiments across five benchmarks and three model scales to validate our theoretical studies. We observe that MAS consistently helps when relays are near-sufficient, especially for weaker models. In contrast, MAS gains shrink or reverse when relays incur information loss, especially for stronger models that can already extract useful information from redundant context and thus gain little from compression. Our study shows that multi-agent design is fundamentally an information-bottleneck optimization problem. This perspective explains when bounded inter-agent communication helps or hurts.

  \vspace{5mm}

  % $^{*}$ \textit{Equal Contribution}
  
  % $^{\coloremojicode{2709}}$ \textit{Corresponding Author}

  \vspace{1mm}
  \vspace{6mm}

   \textbf {* These authors contributed equally}
   
   % \textbf {$\dag$ These authors jointly led the project}

  \vspace{3mm}

  \github{} \textbf{Github Repository}: 
  \url{https://github.com/divelab/MAS-SAS}

  % \coloremojicode{1F4E7} \textbf{Correspondence}: \href{}{siqisun@fudan.edu.cn}

    % \coloremojicode{1F4E7} \textbf{Contact}: \href{}{ \textbraceleft weijiaqi, yangyuejin\textbraceright@pjlab.org.cn}

    \vspace*{0.2in}

\end{abstract}

\maketitle

% \vspace{3mm}
\pagestyle{headstyle}
\thispagestyle{empty}

\newpage
\vspace{2em}
\tableofcontents

\newpage

\section{Introduction}

Multi-agent systems (MAS) based on large language models (LLMs) have emerged as a promising paradigm for complex tasks~\cite{chen2024survey, guo2024large, tran2025multi}. By decomposing tasks into subtasks and assigning them to specialized agents, MAS can support role specialization, modular reasoning, and structured collaboration~\citep{chen2023gamegpt, wang2025tdag,hong2024metagpt, li2023camel}. MAS has shown empirical benefits across domains such as software engineering, mathematical reasoning, scientific discovery~\cite{zhou2025autonomous}, and planning~\cite{hong2024metagpt, zhou2025toward, choi2025atlas, du2024improving}.
However, the advantage of MAS over single-agent systems (SAS) remains unclear, with performance varying inconsistently across settings. Recent studies show that MAS gains diminish as the base model capability improves~\citep{huang2025single,kim2025towards}, or gains are attributable to increased compute rather than architectural advantages~\citep{kapoor2025ai, han2025single,li2026single}. These results suggest that MAS is not inherently superior, but helpful only under certain conditions. The existing comparative studies are mostly empirical, without explaining the underlying reason that makes MAS beneficial in some settings and harmful in others. This raises a fundamental question: \emph{what structural mechanism governs when MAS help over SAS?}

\begin{wrapfigure}{r}{0.4\textwidth}
    % \vspace{-0.8em}
    \centering
    \includegraphics[width=0.95\linewidth]{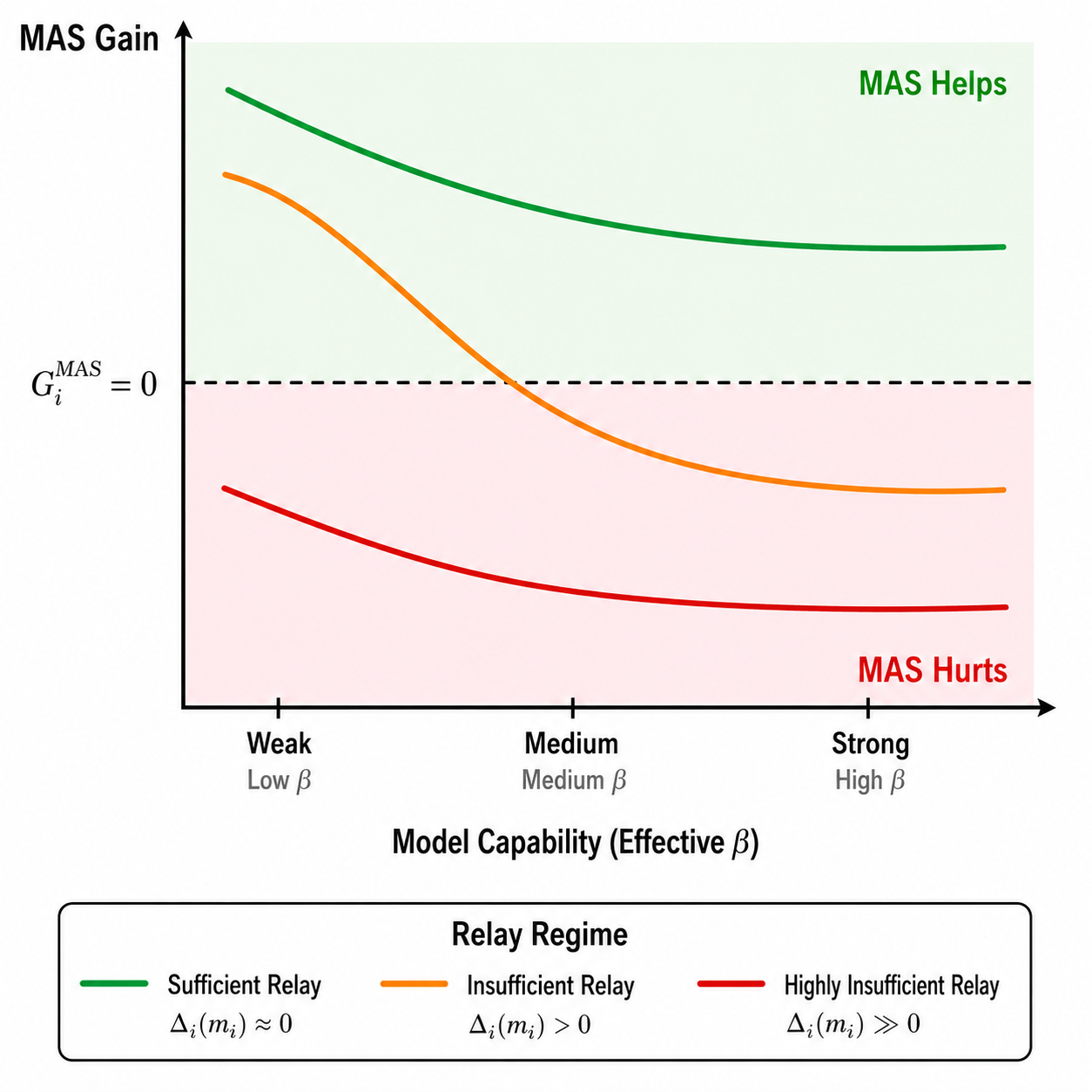}
    \vspace{-0.6em}
    \caption{Predicted MAS gain by relay regime. MAS helps when context reduction outweighs capability-weighted relay loss; stronger models gain less from compression and lose more from information loss.}
    \label{fig:relay-regimes}
    \vspace{-0.6em}
\end{wrapfigure}

Here, we propose an information flow perspective to study the essential difference between SAS and MAS, which may not lie simply in the number of agents, but in how information is organized and transmitted~\cite{shen2025understanding, zhu2025multiagentbench,ao2026reliability}. A SAS maintains a single shared context in which all subtasks and intermediate reasoning traces are accumulated. In contrast, a MAS uses isolated local contexts connected by relay messages. Each relay acts as a compressed interface: it can reduce downstream interference by suppressing noisy contexts, which may cause information loss that later workers cannot recover. These competing effects are naturally captured by the \emph{information bottleneck} principle \citep{tishby1999information, alemi2017deep, xu2021information}. From this perspective, MAS design can be viewed as an optimization problem over relay compression: how can upstream-irrelevant context be removed while preserving information relevant for downstream tasks?

% We show that MAS and SAS are equivalent under unrestricted relay communication. A MAS can simulate a SAS by passing the full upstream context, so the nontrivial MAS effect arises from bounded relay compression rather than decomposition alone. As illustrated in Figure~\ref{fig:intro-overview} (b), we model this compression as an \textit{information bottleneck~(IB)} with an effective parameter $\beta$ that captures downstream LLM-model capability (details in Section~\ref{sec:ib}). This yields a stage-wise MAS gain: MAS gains from upstream context reduction, but pays a cost for relay information loss. Formally, MAS is beneficial when the context-reduction benefit exceeds the capability-weighted relay-loss cost. As $\beta$ increases with base LLM-model capability, stronger models become less helped by compression and more sensitive to information loss. Figure~\ref{fig:relay-regimes} summarizes this predicted trade-off across relay regimes.
We show that MAS and SAS are equivalent under unrestricted relay communication. A MAS can simulate a SAS by passing the full upstream context, so the nontrivial MAS effect arises from bounded relay compression rather than decomposition alone. Figure~\ref{fig:relay-regimes} summarizes the high-level prediction of our framework: MAS helps when context reduction dominates relay information loss, but its gain shrinks or reverses when relays discard downstream-relevant information. To explain this trade-off, we model each relay as an \textit{information bottleneck~(IB)}, as illustrated in Figure~\ref{fig:intro-overview}(b), with an effective parameter $\beta$ that captures downstream LLM capability (details in Section~\ref{sec:ib}). This yields a stage-wise MAS gain: MAS is beneficial when the context-reduction benefit exceeds the capability-weighted relay-loss cost. As $\beta$ increases with base LLM capability, stronger models become less helped by compression and more sensitive to information loss.

We validate our theoretical investigations with 18 empirical studies, spanning five agentic benchmarks, including ALFWorld~\citep{shridhar2020alfworld}, WebShop~\cite{yao2022webshop}, WorkBench~\cite{styles2024workbench}, WideSearch~\cite{wong2025widesearch}, and TravelPlanner~\cite{xie2024travelplanner}, and three model scales, including Qwen2.5-7B, GPT-4o-mini, Qwen3.5-27B. As shown in Figure~\ref{fig:intro-overview} (a), we carefully design three settings to separate the effect of decomposition from relay-based context isolation. Across five benchmarks and three model scales, MAS consistently improves over SAS-contextflow when relays are near-sufficient, while its gains diminish or reverse when relays lose downstream-relevant information, especially for stronger models.

Our contributions can be summarized as follows:
\begin{itemize}[leftmargin=*, topsep=0pt, noitemsep]
    \item \textbf{An information-bottleneck framework for MAS.} We formulate MAS design as relay compression between isolated worker contexts. We show that, with unbounded relay communication, any SAS can be simulated exactly by a MAS, showing that nontrivial MAS gains and losses arise from bounded relay compression.

    \item \textbf{A stage-wise MAS gain.} From the information-bottleneck objective, we derive a per-stage MAS gain showing that MAS helps when the benefit of upstream noisy context reduction outweighs the cost of relay information loss. The effective parameter $\beta$ captures downstream model capability, explaining why the same relay can benefit weaker models but hurt stronger ones.

    \item \textbf{Controlled empirical validation.}  We conduct 18 comparisons across five benchmarks and three model scales. 
    By comparing SAS, SAS-contextflow, and MAS, we isolate relay compression from subtask decomposition and show that observed MAS gains follow the predicted relay-loss regimes. Our framework suggests how to design a better MAS by optimizing relays to remove upstream-irrelevant context while preserving downstream-relevant information.
\end{itemize}

\begin{figure}[t]
    \centering
    \makebox[\textwidth][c]{%
        \includegraphics[width=1.0\textwidth]{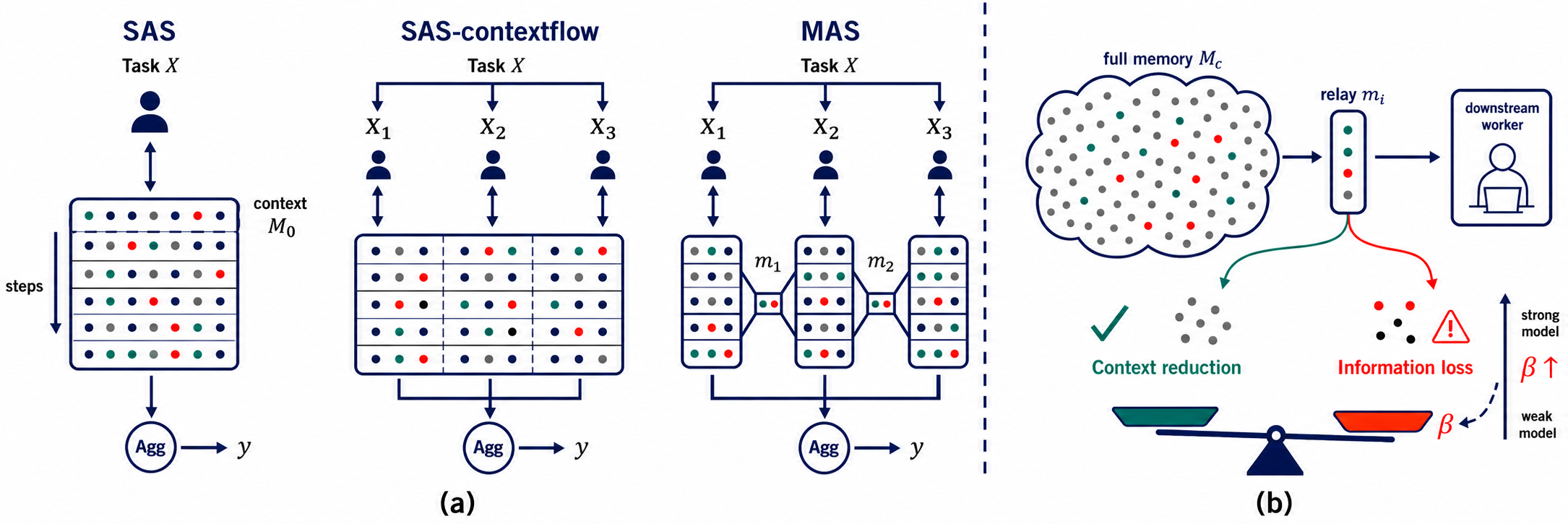}%
    }
    \vspace{-15pt}
    \caption{\textbf{Information-flow view of MAS vs. SAS.} \textbf{(a)} Three controlled prototypes for task $X$: \textit{SAS:} uses single shared context; \textit{SAS-contextflow} follows a planner-induced subtask sequence $X_1, X_2, X_3$ with shared context; \textit{MAS} uses the same decomposition but connects workers through compressed relays $m_1, m_2$. \textbf{(b)} Each relay compresses $M_i$ into $m_i$, trading off context reduction against information loss, with an effective $\beta$ that grows with model capability.}
    \label{fig:intro-overview}
\end{figure}

\section{Related Work}

\textbf{Multi-Agent Systems.} 
LLM-based MAS decomposes complex tasks across specialized agents that communicate in natural language. 
Early frameworks introduce role-playing, workflow execution, and conversational protocols~\citep{hong2024metagpt,li2023camel,wu2024autogen}, and follow-up work demonstrates gains on software development, reasoning, scientific discovery, and simulation~\citep{du2024improving,qian2024chatdev,liang2024encouraging,park2023generative}.
More recent studies examine how performance changes as the number of agents grows, reporting both improvements and non-monotonic trends~\citep{qian2025scaling,chen2024agentverse,li2024more,chen2024more,riedl2026emergent}.
Broader coverages are given in~\citep{guo2024large,zhang2025survey}.

\noindent\textbf{Empirical Studies of MAS versus SAS.}
Recent work consistently shows that the advantages of MAS depend on tasks and models, and that gains diminish as the base LLM become more powerful~\citep{huang2025single,kim2025towards,kapoor2025ai,li2026single}.
Along a similar line to our work, it has been shown in \citep{han2025single} that, under matched thinking-token budgets, a SAS can match or outperform MAS on multi-hop reasoning, and they attribute the gap to communication bottlenecks. 
This observation is consistent with our relay information loss term $\Delta_i(m_i)$, but it only describes one side of the MAS trade-off. 
Our framework additionally captures the benefit of upstream context reduction $H(M_i \mid m_i)$ and the capability-dependent weight $\beta$ that decides which term wins. 
We therefore move from empirical description to a mechanistic explanation of when MAS helps.

\noindent\textbf{Information Bottleneck.} 
The idea of information bottleneck (IB)~\citep{tishby1999information} formalizes the trade-off between compression and relevance preservation. It has been used to study deep representation learning~\citep{alemi2017deep,tishby2015deep} and relay channels in communication theory~\citep{xu2021information,steiner2020broadcast}. Our work represents an attempt to use it in LLM agent systems by modeling each inter-worker interface as a relay bottleneck and deriving a stage-wise decomposition of MAS gain that explains the inconsistent empirical findings discussed above.
% Extended discussion of MAS architectures and IB variants appears in Appendix~\ref{app:related-work}.

\section{Problem Setup and MAS-SAS Equivalence}
\label{sec:setup}

% \yuchao{Before going to subsection, I think it is better to specify why we give this definition, and it will end up using in what place (or the reason to formulate it). E.g., you can say "To ..., we formulate ... for ...". Also, do these definition also appear before in other papers? If so, it is better to connect with them for reviewers to get in}

% This section formalizes LLM-based agent systems and establishes the theoretical foundations for understanding decomposition.
% We first define what is agentic tasks and agent systems, then we introduce a xxx that

This section introduces a formal setup for task structure, agent systems, and inter-agent communication.
We model agentic tasks as partially observable sequential decision processes and define both single-agent systems and multi-agent systems using a common basic-agent abstraction. 
Then, we introduce the infinite-bandwidth MAS--SAS equivalence to show why compressed relay communication is the key source of MAS gains and losses.

%\subsection{Agentic Tasks}

\textbf{Agentic Tasks.} We formalize agentic tasks as partially observable sequential decision processes. 
An agentic task class is defined as
$\Gamma \triangleq \langle \mathcal{S}, \mathcal{A}, \mathcal{O}, P, \Omega, \mathcal{X} \rangle,$
where $\mathcal{S}$, $\mathcal{A}$, and $\mathcal{O}$ denote the state, action, and observation spaces, respectively,
$P : \mathcal{S} \times \mathcal{A} \to \Delta(\mathcal{S})$ is the transition kernel,
$\Omega : \mathcal{S} \times \mathcal{A} \to \Delta(\mathcal{O})$ is the observation kernel,
and $\mathcal{X}$ is the task-instance space.
% where $I$ is the initial instruction, $s_0\in\mathcal S$ is the initial state, and $U_X$ is the associated utility function. \wendi{may need to delete s0}

%\subsection{Agent Systems}

\textbf{Basic Agent.}
We define a basic agent as
$\mathsf{Agent} \triangleq \langle \mathcal{T}, \mathcal{M}, \pi_T, \pi_A, \mathrm{Agg} \rangle$,
where $\mathcal{T}$ is the thought space, $\mathcal{M}$ is the context space, $\pi_T$ and $\pi_A$ are the thought and action policies, and $\mathrm{Agg}$ is the output function. Here, context refers to the accumulated execution history visible to the agent, consisting of the thoughts, actions, and observations from previous steps. Given a specific task instance $X$ and initial observation $o_0$, the agent initializes
$M_0 = (X, o_0)$ and updates its context by
$
M_{t+1} = M_t \oplus (\tau_t, a_t, o_t),
$
where $\tau_t \sim \pi_T(\cdot \mid M_t)$ is the thought,
$a_t \sim \pi_A(\cdot \mid M_t, \tau_t)$ is the action, and $o_t$ is the resulting observation under $(P,\Omega)$. The agent terminates when $\pi_A$ emits a terminal action, producing context $M_T$ and output $\mathrm{Agg}(M_T)$.

\textbf{Single-Agent System (SAS).}
A SAS instantiates one basic agent on the task instance $X$, with initial context $M_0 = (X, o_0)$. It produces a single context $M_{\mathrm{SAS}}$ and final output $y_{\mathrm{SAS}} = \mathrm{Agg}(M_T)$ along a single context trajectory.

\textbf{Multi-Agent System (MAS).}
A MAS consists of a planner, a sequence of $n$ workers, and a system-level aggregator. The \textbf{planner} decomposes the instance \(X\) into an ordered sequence of
sub-instances \(X_1,\ldots,X_n\). Each sub-instance \(X_i\) induces a
worker-specific target variable \(Y_i\), which represents the task-relevant
outcome that worker \(i\) is expected to infer or
produce. We focus on instance-level decomposition, where the task kernels $(P, \Omega)$ remain fixed and only $X$ is partitioned across workers. Each \textbf{worker} is a basic agent operating on $X_i$. Inside a worker, context propagates as in a basic agent. Between workers, the sequence forms a chain that communicates through \textbf{relays}: worker $i$ produces a relay $m_i = \mathrm{Agg}_i(M_i)$ that is passed to worker $i+1$, with $m_0 = \varnothing$. Worker $i$ initializes its context as $M_0^{(i)} = (X_i, o_0^{(i)}) \oplus m_{i-1}$. The \textbf{system-level aggregator} then combines all terminal memories into the final output $y_{\mathrm{MAS}} = \mathrm{Agg}_{\mathrm{MAS}}(M_1, \dots, M_n)$.

%\subsection{Infinite-Bandwidth MAS-SAS Equivalence}
%\label{sec:equivalence}

\textbf{Infinite-Bandwidth MAS-SAS Equivalence.} A MAS differs from a SAS in two structural ways: a planner partitions the task $X$ into subtasks $X_1, \dots, X_n$, and workers communicate through the relays $m_1, \dots, m_{n-1}$. Let $B_{\mathrm{sys}}$ denote the \emph{inter-worker bandwidth}, defined as the maximum entropy any single relay $m_i$ can carry, so that $H(m_i) \leq B_{\mathrm{sys}}$ for all $i$. We now show that as this bandwidth grows without bound, any SAS can be simulated by a MAS. The proof is deferred to Appendix~\ref{app:equivalence-proof}.

\begin{propositionbox}{Infinite-Bandwidth MAS--SAS Equivalence}{equivalence}
Assume the SAS context trajectory has finite entropy at every step. Then as
$B_{\mathrm{sys}} \to \infty$, for any SAS there exists a MAS with
$y_{\mathrm{MAS}} = y_{\mathrm{SAS}}$.
\end{propositionbox}

In practice, however, $B_{\mathrm{sys}}$ is bounded. LLM agents communicate through token-limited channels far smaller than $H(M_i)$, forcing each relay $m_i$ to be a compressed summary of the upstream context. Proposition~\ref{prop:equivalence} therefore shows relay compression is the source of the empircal variability of MAS over SAS. Whether this compression helps or hurts depends on what the relay retains versus discards relative to the downstream subtask.

\section{MAS Relay Compression as an Information Bottleneck}
\label{sec:ib}

% \yuchao{The commented out textbf should be important for the reader to know the storyline}

% Section~\ref{sec:equivalence} located the structural source of MAS-SAS variability in how each relay compresses the upstream context. By Proposition~\ref{prop:equivalence}, a relay that transmits the upstream context in full exactly recovers SAS. MAS performance is therefore shaped largely by how the relay compression is designed at each interface. In this section we show that relay compression among MAS workers is actually an optimization over conditaional infromation bottleneck. Based on this, we can derive the MAS gain over SAS.

It has been established above that full-context relay communication allows MAS to recover SAS exactly. 
This raises a natural question: \textbf{what governs MAS performance when relays are bounded?} 
In this section, we show that relay design is an \textit{information-bottleneck} optimization problem. 
A relay helps by removing upstream noisy context, but hurts when it loses information needed by downstream workers. 
This trade-off leads to a MAS gain that characterizes when MAS improves over SAS.

\subsection{Relay Compression as an Information Bottleneck}
\label{sec:ib_explains}
Consider two consecutive workers $i$ and $i+1$ under the planner-induced decomposition defined in Section~\ref{sec:setup}. 
Worker $i$ has context $M_i$ and produces a relay message $m_i=\mathrm{Agg}_i(M_i)$. 
Worker $i+1$ receives $m_i$ together with its local subtask $X_{i+1}$, which includes its sub-instruction and initial observation. 
Let $Y_{i+1}$ denote the downstream target associated with this subtask. 
Since $m_i$ is compressed from $M_i$, it cannot increase the information about $Y_{i+1}$ beyond that contained in $M_i$.
Therefore, once $M_i$ and $X_{i+1}$ are given, $m_i$ provides no additional information about $Y_{i+1}$. A formal justification is provided in Appendix~\ref{app:markov}. This gives the Markov relation
\begin{equation}
\label{eq:markov-relation}
m_i \leftrightarrow M_i \leftrightarrow Y_{i+1} \mid X_{i+1}.
\end{equation}
\textbf{IB Objective in the MAS Setting.}
A relay should compress upstream context while preserving the information needed for the downstream subtask.
This trade-off can be formulated through the information bottleneck (IB) objective
\begin{equation}
\label{eq:ib-objective}
\min_{p(m_i \mid M_i)}
\;
I(M_i;m_i)
-
\beta I(m_i;Y_{i+1}\mid X_{i+1}),
\end{equation}
where \(I(\cdot;\cdot)\) denotes mutual information. The first term measures how much upstream information is retained in the relay \(m_i\), while the second term measures how much downstream-relevant information about \(Y_{i+1}\) is preserved given the local subtask \(X_{i+1}\). Although not optimized explicitly, this objective is implicitly induced by the MAS design, (e.g., prompts, communication, and model capability) and characterizes the optimal relay.

\textbf{Effective $\beta$ and LLM Capability.}
In the IB objective, $\beta>0$ controls the preference between compression and predictive information preservation. 
In MAS, we interpret $\beta$ as an effective parameter induced by downstream LLM capability. 
Weaker models are more sensitive to noisy context and thus favor stronger compression (smaller effective $\beta$), whereas stronger models can better utilize rich context and therefore place higher value on preserving predictive information (larger effective $\beta$). 
This could explain why the same relay may benefit weaker models but degrade performance for stronger ones.

\textbf{SAS as the No-Compression Relay.}
Proposition~\ref{prop:equivalence} shows that a MAS recovers SAS when relay communication is unbounded. In the IB formulation, this corresponds to the no-compression relay \(m_i=M_i\). Substituting \(m_i=M_i\) into Eq.~\eqref{eq:ib-objective} gives
\begin{equation}
\label{eq:sas-loss}
\mathcal{L}_i^{\mathrm{SAS}}(\beta)
\triangleq
\mathcal{L}_i^{\mathrm{MAS}}(m_i=M_i;\beta)
=
H(M_i)
-
\beta I(M_i;Y_{i+1}\mid X_{i+1}),
\end{equation}
Therefore, under the same planning decomposition, SAS corresponds to the no-compression point in the relay design space, while MAS corresponds to compressed relay designs.

\subsection{MAS Gain: Trade-off in Relay Compression}

The IB objective shows that a MAS relay must balance two effects: compressing upstream information and preserving downstream-relevant information. 
To isolate the second effect, we define the \textbf{relay information loss} as follows.

\begin{definitionbox}{Relay Information Loss}{relay-info-loss}
The relay information loss induced by $m_i$ is
\begin{equation}
\label{eq:relay-info-loss}
\Delta_i(m_i)
\triangleq
I(M_i;Y_{i+1}\mid X_{i+1})
-
I(m_i;Y_{i+1}\mid X_{i+1}).
\end{equation}
\end{definitionbox}
This quantity measures how much target-relevant information is lost when the full upstream context $M_i$ is compressed into the relay $m_i$. 
By the Markov relation in Eq.~\eqref{eq:markov-relation}, we have $\Delta_i(m_i)\ge 0$. 
This gives the following interpretation: When $\Delta_i(m_i)=0$, the relay is sufficient for the next worker, meaning that it preserves the information in $M_i$ needed to predict $Y_{i+1}$ given $X_{i+1}$; When $\Delta_i(m_i)>0$, the relay is insufficient because it omits information needed by the next worker, and larger values indicate more severe relay information loss.

Using the relay information loss, the MAS IB objective can be rewritten as 
\begin{equation}
\label{eq:mas_loss}
\mathcal{L}_i^{\mathrm{MAS}}(m_i;\beta)
=
I(M_i;m_i)
+
\beta \Delta_i(m_i)-
\beta I(M_i;Y_{i+1}\mid X_{i+1}).
\end{equation}
The last term \(I(M_i;Y_{i+1}\mid X_{i+1})\) is independent of the relay design.

\textbf{MAS Gain.}
Building on the MAS and SAS IB objective, we define the local MAS gain as the reduction in the relay objective relative to the no-compression SAS baseline:
\begin{equation}
\label{eq:mas-gain-def}
G_i^{\mathrm{MAS}}
\triangleq
\mathcal L_i^{\mathrm{SAS}}
-
\mathcal L_i^{\mathrm{MAS}}.
\end{equation}

% \begin{theorembox}{MAS Gain Decomposition}{mas-gain-decomposition}
% Given the no-compression SAS baseline $m_i=M_i$, the local MAS gain decomposes into a benefit term from upstream context reduction and a cost term from relay information loss:
% \begin{equation}
% \label{eq:mas-gain}
% \bm{G_i^{\mathrm{MAS}}}
% =
% \underbrace{\textcolor{gainGreen}{\bm{H(M_i \mid m_i)}}}_{\text{upstream context reduction}}
% -
% \underset{
% \substack{
% \uparrow\\[-0.5pt]
% \text{\scriptsize model capability}
% }
% }{\bm{\beta}}
% \underbrace{\textcolor{lossRed}{\bm{\Delta_i(m_i)}}}_{\text{relay information loss}}.
% \end{equation}
% \end{theorembox}

\begin{theorembox}{MAS Gain Decomposition}{mas-gain-decomposition}
Given the no-compression SAS baseline $m_i=M_i$, the local MAS gain decomposes into a benefit term from upstream context reduction and a cost term from relay information loss
\begin{equation}
\label{eq:mas-gain}
\bm{G_i^{\mathrm{MAS}}}
=
\underbrace{\bm{H(M_i \mid m_i)}}_{\text{upstream context reduction}}
-
\underset{
\substack{
\uparrow\\[-0.5pt]
\text{\scriptsize model capability}
}
}{\bm{\beta}}
\underbrace{\bm{\Delta_i(m_i)}}_{\text{relay information loss}}.
\end{equation}
\end{theorembox}
% \yuchao{Maybe "upstream context reduction" -> "upstream context reduction"}
Theorem~\ref{theorem:mas-gain-decomposition} shows that MAS gain under the same planning decomposition is governed by a trade-off between upstream context reduction and relay information loss. 
The first term, \(H(M_i\mid m_i)\), measures the amount of upstream information removed by relay compression, while the second term, \(\beta \Delta_i(m_i)\), measures the corresponding loss of downstream-relevant information. 
Therefore, MAS is locally beneficial exactly when
\begin{equation}
\label{eq:mas-advantage-condition}
G_i^{\mathrm{MAS}} > 0
\quad \Longleftrightarrow \quad
H(M_i \mid m_i) > \beta \Delta_i(m_i).
\end{equation}
This trade-off is capability-dependent: weaker models benefit more from reducing redundant context, whereas stronger models can better utilize redundant upstream information, reducing the relative benefit of relay compression. Overall, this decomposition explains the structural mechanism behind when and why MAS outperforms SAS.

\paragraph{Local-to-Downstream Sufficiency.}
The gain decomposition above is local to a single relay interface. In a sequential MAS, local sufficiency for \(Y_{i+1}\) does not automatically imply sufficiency for later targets. It holds for a later target \(Y_j\) only when all information from \(M_i\) that is needed for \(Y_j\) has already been captured by the next target \(Y_{i+1}\). In that case, a locally sufficient relay \(m_i^*\) also satisfies
\begin{equation}
I(Y_j;M_i\mid m_i^*,X_{i+1},X_j)=0,\quad j>i+1.
\end{equation}
Otherwise, if \(Y_j\) still depends on details in \(M_i\) that are not reflected in \(Y_{i+1}\), hidden long-range dependencies may make a locally sufficient relay insufficient for later workers. See Appendix~\ref{app:local-to-downstream} for details.

\section{Experiments}
% \xj{todo: add experimental details}
% \subsection{Empirical Validation of Theorem 4.3}
\label{sec:exp-analysis}

\subsection{Experimental Setup}

% To validate the stage-wise MAS gain across tasks and model scales, we compare three controlled prototypes on five benchmarks under three model capability levels.

\textbf{Benchmarks.}
We validate our framework on five agentic benchmarks that cover diverse task structures and relay-complexity regimes:
ALFWorld~\cite{shridhar2020alfworld}, WebShop~\cite{yao2022webshop}, WorkBench~\citep{styles2024workbench}, WideSearch~\citep{wong2025widesearch}, and TravelPlanner~\cite{xie2024travelplanner}.
For empirical analysis, we use $\delta$ to denote the estimated relay complexity of each benchmark, serving as a proxy for how difficult it is for a relay to be sufficient.
For a relay from worker $i$ to worker $i+1$, let $m_i^\star$ be a locally sufficient relay
and define its stage-wise complexity as
% \begin{equation}
% \delta_i \triangleq L(m_i^\star),
% \end{equation}
\begin{equation}
\delta_i
\triangleq
\min_{m_i} L(m_i)
\quad \text{s.t.} \quad
\Delta_i(m_i)=0.
\end{equation}
where $L(\cdot)$ is the relay description length, so $\delta$ denotes the minimum description length of a downstream-sufficient relay. We use $\delta$ as the benchmark-level average of $\delta_i$ across subtasks. We assign each benchmark to a predicted $\delta$ regime based on its task information structure, as summarized in Table~\ref{tab:dataset-relay-regimes}.
For limited-capability downstream models, larger $\delta$ implies larger expected relay information loss: $\delta \approx 0$ when downstream-relevant information is compactly describable, $\delta > 0$ when detailed upstream evidence is needed, and $\delta \gg 0$ in cases where sufficiency requires many exact task-specific details, such as prior states, command outputs, or execution histories.

TravelPlanner contains both regimes within the same task: commonsense checks evaluate mostly local phase validity and therefore require low-complexity relays, whereas hard constraints require global plan consistency and therefore induce higher relay complexity.
We therefore report TravelPlanner-CS ($\delta \approx 0$) and TravelPlanner-HC ($\delta > 0$) separately.
We report Success Rate (SR) for ALFWorld and WorkBench, SR and Average Reward for WebShop, Item~F1 and Row~F1 for WideSearch, and SR, Commonsense and Hard-constraint Macro Pass rates for TravelPlanner.
Detailed benchmark descriptions and relay-complexity justifications are provided in Appendix~\ref{app:benchmarks}.

\begin{table*}[t]
\centering
\small
\caption{\textbf{Estimated relay-complexity regime for each benchmark.}
$\delta$ denotes the estimated relay complexity, defined as the minimum description length of a downstream-sufficient relay. 
The regime ($\approx 0$, $> 0$, or $\gg 0$) is assigned based on each task's information structure.}
\resizebox{\textwidth}{!}{
\begin{tabular}{llc}
\toprule
\textbf{Dataset / Metric} & \textbf{Relay property} & \textbf{Estimated relay complexity $\delta$} \\
\midrule
ALFWorld 
& Needs only current object state and location. 
& \textcolor{relayGreen}{$\delta \approx 0$} \\
WideSearch 
& Subqueries are mostly independent. 
& \textcolor{relayGreen}{$\delta \approx 0$} \\
TravelPlanner-CS
& Commonsense checks are mostly local to each worker.
& \textcolor{relayGreen}{$\delta \approx 0$} \\
WebShop 
& Purchase decisions require accumulated product and navigation evidence. 
& \textcolor{relayOrange}{$\delta > 0$} \\
WorkBench 
& Downstream actions require exact prior states and command outputs. 
& \textcolor{relayRed}{$\delta \gg 0$} \\
TravelPlanner-HC
& Hard constraints require global evidence and cross-stage consistency. 
& \textcolor{relayOrange}{$\delta > 0$} \\
\bottomrule
\end{tabular}
}
\label{tab:dataset-relay-regimes}
\vspace{-0.8em}
\end{table*}

\textbf{Prototypes.}
To separate planner-induced subtask structure from relay-based context compression, we compare three controlled prototypes, as illustrated in Figure~\ref{fig:intro-overview}(a). 
\textbf{SAS} is a standard single-agent system. 
A single agent receives the original task instruction and carries all intermediate reasoning, actions, and observations in one accumulated context until producing the final output. 
\textbf{SAS-contextflow} is a single-agent system that follows the same planner-induced subtask sequence and role assignments as MAS, while keeping the accumulated context fully visible across subtasks.
Each step is conditioned on the current sub-instruction together with the complete upstream context. 
\textbf{MAS} uses the same planner-induced subtask sequence as SAS-contextflow, but assigns each subtask to a separate worker with an isolated local context.
Worker $i$ receives its sub-instruction and incoming relay $m_{i-1}$, compressed from the upstream context.
Thus, MAS differs from SAS-contextflow only by replacing full accumulated context with compressed relay communication, which directly reflects the trade-off between context reduction and relay information loss.
To ensure a fair comparison, all prototypes are evaluated under the same maximum total step budget.

\textbf{Models.}
We evaluate all three prototypes under three model scales to test how capability modulates the relay compression trade-off: \textbf{Qwen2.5-7B-Instruct} (weak), \textbf{GPT-4o-mini} (medium), and \textbf{Qwen3.5-27B} (strong). 
For Qwen3.5-27B, we use the AWQ 4-bit quantized checkpoint
(\textbf{Qwen3.5-27B-AWQ-4bit}) for efficient deployment.
This ordered capability axis allows us to test how the same relay-complexity regime interacts with model capability: context reduction is expected to dominate when relays are low-complexity, whereas relay information loss becomes more limiting as relay complexity increases. Detailed experimental settings are provided in Appendix~\ref{app:setting}.

Our experiments are designed to validate Theorem~\ref{theorem:mas-gain-decomposition} through two research questions: 
\textbf{(RQ1)} Do low-complexity relays ($\delta \approx 0$) yield consistent MAS gains over the SAS-contextflow baseline, as expected when the context-reduction benefit outweighs the capability-weighted relay-loss cost?
\textbf{(RQ2)} Do higher-complexity relays ($\delta > 0$) cause MAS gains to shrink or reverse when bounded communication cannot preserve all downstream-relevant information?

%---------------------------------------------------------------
\subsection{Main Results}
%---------------------------------------------------------------

\begin{table}[!t]
\centering
\footnotesize
\setlength{\tabcolsep}{4pt}
\renewcommand{\arraystretch}{0.92}
\setlength{\abovecaptionskip}{2pt}
\setlength{\belowcaptionskip}{3pt}
\caption{MAS gains over SAS-contextflow, grouped by estimated relay complexity $\delta$.}
\label{tab:relay-regime-summary}
\resizebox{0.74\textwidth}{!}{%
\begin{tabular}{lcccc}
\toprule
\textbf{Task / Metric} & \textbf{Estimated relay complexity $\delta$} & \textbf{Qwen2.5-7B} & \textbf{GPT-4o-mini} & \textbf{Qwen3.5-27B} \\
\midrule
ALFWorld & \textcolor{relayGreen}{$\delta \approx 0$} & \gainpos{+0.194} & \gainpos{+0.157} & \gainpos{+0.023} \\
WideSearch & \textcolor{relayGreen}{$\delta \approx 0$} & \gainpos{+0.079} & \gainpos{+0.063} & \gainpos{+0.028} \\
TravelPlanner-CS & \textcolor{relayGreen}{$\delta \approx 0$} & \gainpos{+0.011} & \gainpos{+0.183} & \gainpos{+0.028} \\
\midrule
WebShop & \textcolor{relayOrange}{$\delta > 0$} & \gainpos{+0.080} & \gainpos{+0.086} & \gainneg{-0.003} \\
WorkBench & \textcolor{relayRed}{$\delta \gg 0$} & \gainneg{-0.005} & \gainneg{-0.086} & \gainneg{-0.014} \\
TravelPlanner-HC & \textcolor{relayOrange}{$\delta > 0$} & \gainpos{+0.017} & \gainpos{+0.161} & \gainneg{-0.233} \\
\bottomrule
\end{tabular}%
}
\end{table}

\begin{figure*}[!t]
\centering
\includegraphics[width=0.95\textwidth]{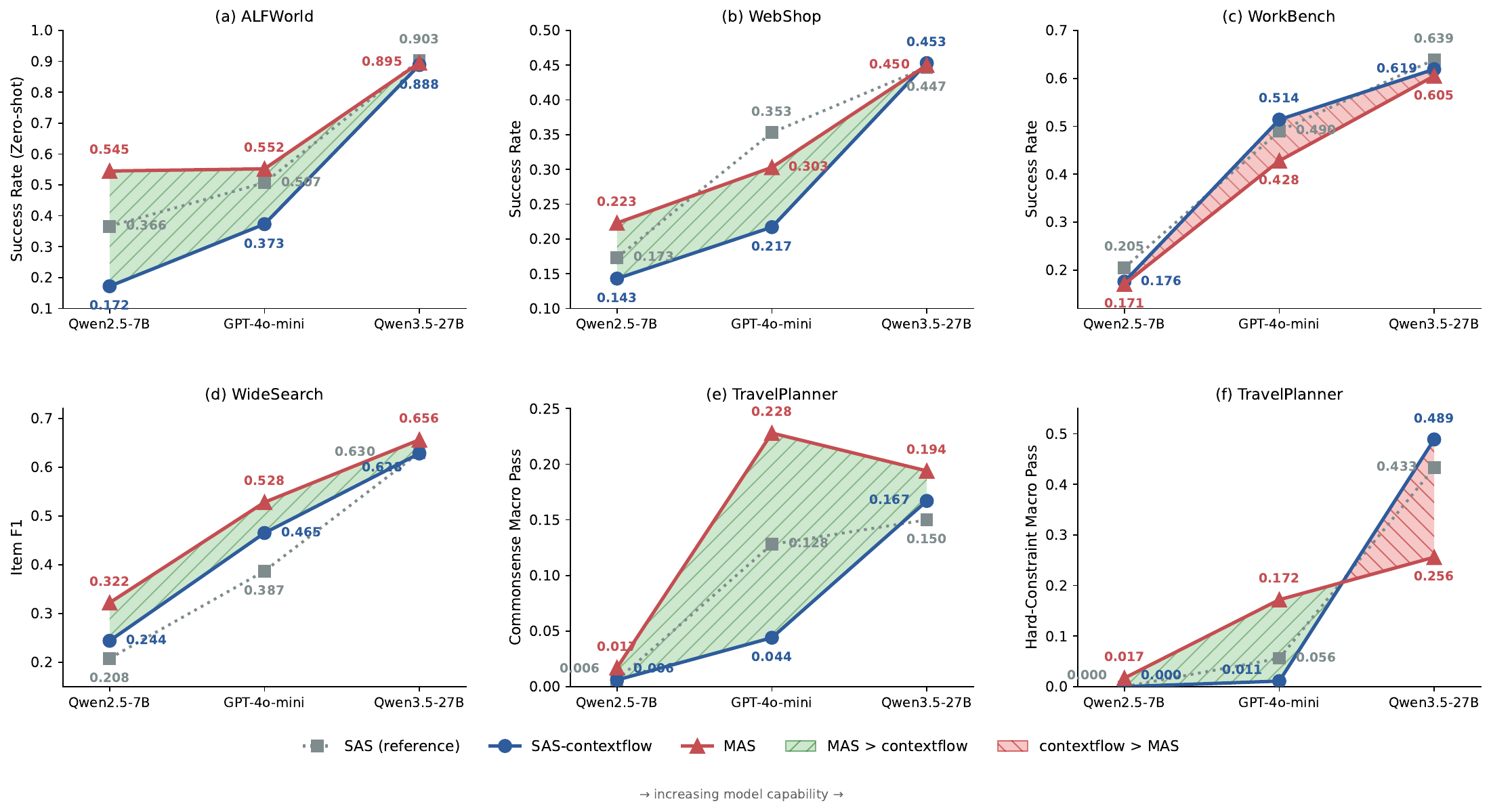}
\vspace{-0.3em}
\caption{\textbf{Per-task performance across three model scales.} Capability increases from left to right. Green and red shading indicate where MAS outperforms or underperforms SAS-contextflow.}
\label{fig:all_datasets}
\vspace{-1.2em}
\end{figure*}

\textbf{Main Results.}
Figure~\ref{fig:all_datasets} plots MAS, SAS-contextflow, and SAS as trajectories over increasing model capability, with green shading where MAS leads SAS-contextflow, and red shading where SAS-contextflow leads.
To connect these results to our theory, Table~\ref{tab:relay-regime-summary} groups each benchmark or metric by its estimated relay complexity $\delta$ and reports the empirical trend.
Detailed per-metric results are provided in Appendix~\ref{app:results}.
We organize the analysis below by estimated relay-complexity regime.

\textbf{RQ1: Low-complexity-relay Benchmarks ($\delta \approx 0$): ALFWorld, WideSearch, and TravelPlanner-CS.}
These three benchmarks fall into the low-complexity relay regime ($\delta \approx 0$) because their natural decompositions admit compact sufficient relays. 
WideSearch is a structured information-retrieval task whose subqueries are parallel and largely independent, so no cross-worker relay is required. 
In ALFWorld, our MAS design uses an explorer to locate the target object and an executor to complete the action sequence, so the relay only needs to provide the object location. 
TravelPlanner-CS evaluates local commonsense feasibility, such as transportation, accommodation, dining, and attraction validity. 
As shown in Figure~\ref{fig:all_datasets}, MAS is competitive with or stronger than SAS on these benchmarks, indicating that the decomposition is well aligned with the task structure. 
More importantly, under the controlled comparison with SAS-contextflow, MAS yields consistently positive gains: $+0.194$, $+0.157$, and $+0.023$ on ALFWorld; $+0.079$, $+0.063$, and $+0.028$ on WideSearch Item~F1; and $+0.011$, $+0.183$, and $+0.028$ on TravelPlanner-CS, for Qwen2.5-7B, GPT-4o-mini, and Qwen3.5-27B, respectively. 
These consistently positive gains support the low-complexity-relay prediction: MAS reduces context interference when downstream-relevant information can be compactly transmitted, while stronger models derive a smaller marginal benefit from such isolation.

\textbf{RQ2: Higher-complexity-relay Benchmarks ($\delta>0$): WorkBench, WebShop, and TravelPlanner-HC.}
These benchmarks instantiate higher-complexity relay regimes because downstream decisions require detailed upstream evidence that bounded relays cannot fully preserve. 
% In WorkBench, shell subtasks mutate the filesystem, and later commands often depend on exact prior states, file paths, command outputs, and environment variables, making downstream-sufficient relays highly complex ($\delta \gg 0$). 
In WorkBench, agents must retrieve cross-domain evidence before executing write actions, and later decisions depend on exact intermediate read results such as object IDs, timestamps, and name-to-email resolutions. making downstream-sufficient relays highly complex ($\delta \gg 0$). 
In WebShop, a search-then-select decomposition allows the relay to summarize candidate products, but final selection still depends on fine-grained attributes, prices, filtering criteria, and navigation history, yielding moderate relay complexity ($\delta > 0$). 
TravelPlanner-HC is similarly high-complexity because hard constraints require global itinerary consistency across subtasks. 
Compared with SAS, MAS is no longer uniformly stronger in these regimes, and the strongest model often favors full shared context, especially on TravelPlanner-HC. 
Under the controlled MAS-vs-SAS-contextflow comparison, WorkBench shows consistently negative gains across all three models ($-0.005$, $-0.086$, $-0.014$), indicating that severe relay information loss dominates the MAS gain trade-off. 
WebShop and TravelPlanner-HC show a capability-dependent reversal: WebShop shifts from $+0.080$ and $+0.086$ on the weak and medium models to $-0.003$ on Qwen3.5-27B, while TravelPlanner-HC shifts from $+0.017$ and $+0.161$ to $-0.233$. 
These positive-to-negative shifts support Theorem~\ref{theorem:mas-gain-decomposition}: weaker models benefit from noisy context reduction, whereas stronger models can better exploit the full context, so the information loss induced by higher-complexity bounded relays eventually dominates.

Overall, Table~\ref{tab:relay-regime-summary} and Figure~\ref{fig:all_datasets} support the estimated relay-regime pattern. 
Sufficient relays yield positive MAS gains across models, usually shrinking with model capability. 
Insufficient relays either remain negative or flip from positive to negative as models become stronger. 
Thus, MAS helps when context reduction dominates, but loses its advantage when relay information loss becomes the limiting factor, especially on globally coupled tasks.

\section{Ablation Studies}
\label{sec:ablation}
% \vspace{-1em}
We present two ablation studies to test the core mechanisms governing MAS gain: relay-induced context isolation (as distinct from decomposition structure), and relay sufficiency (as distinct from relay length).
% \vspace{-1em}

\begin{table*}[!t]
\centering
\small
\setlength{\tabcolsep}{4pt}
\renewcommand{\arraystretch}{1.10}
\caption{Planning ablation on ALFWorld, WebShop, and TravelPlanner. SAS-Plan uses the same decomposed sub-instructions as MAS but maintains a single shared context.}
\label{tab:planner_ablation}
\resizebox{\textwidth}{!}{%
\begin{tabular}{ll ccc ccc ccc}
\toprule
& & \multicolumn{3}{c}{Qwen2.5-7B-Instruct} & \multicolumn{3}{c}{GPT-4o-mini} & \multicolumn{3}{c}{Qwen3.5-27B} \\
\cmidrule(lr){3-5} \cmidrule(lr){6-8} \cmidrule(lr){9-11}
\textbf{Task} & \textbf{Metric} & SAS & SAS-Plan & \textbf{Gap} & SAS & SAS-Plan & \textbf{Gap} & SAS & SAS-Plan & \textbf{Gap} \\
\midrule
ALFWorld & SR & 0.366 & 0.298 & \planlower{$-0.068$} & 0.507 & 0.448 & \planlower{$-0.059$} & 0.903 & 0.709 & \planlower{$-0.194$} \\
\midrule
WebShop & SR & 0.173 & 0.123 & \planlower{$-0.050$} & 0.353 & 0.270 & \planlower{$-0.083$} & 0.447 & 0.466 & \planhigher{$+0.019$} \\
WebShop & Reward & 0.557 & 0.526 & \planlower{$-0.031$} & 0.666 & 0.492 & \planlower{$-0.174$} & 0.725 & 0.752 & \planhigher{$+0.027$} \\
\midrule
TravelPlanner & CS Macro & 0.006 & 0.011 & \planhigher{$+0.005$} & 0.128 & 0.094 & \planlower{$-0.034$} & 0.150 & 0.167 & \planhigher{$+0.017$} \\
TravelPlanner & HC Macro & 0.000 & 0.000 & \planlower{$0.000$}  & 0.056 & 0.022 & \planlower{$-0.034$} & 0.433 & 0.361 & \planlower{$-0.072$} \\
\bottomrule
\end{tabular}
}
\end{table*}

\textbf{Does Planning Structure Alone Explain MAS Gains?} To isolate whether MAS gains arise from relay-induced context isolation or merely from the subtask structure, we introduce \textbf{SAS-Plan}: a single-agent prototype that receives the same decomposed sub-instructions as MAS workers but executes over a single shared context. Under the formulation of Section~\ref{sec:ib}, SAS, SAS-Plan, and SAS-contextflow sit at the no-compression point of the relay design space ($H(M_i\mid m_i){=}0$, $\Delta_i{=}0$).

Table~\ref{tab:planner_ablation} shows that SAS-Plan consistently performs below SAS across all models on ALFWorld ($-0.068$, $-0.059$, and $-0.194$ for Qwen2.5-7B, GPT-4o-mini, and Qwen3.5-27B). On WebShop and TravelPlanner, SAS-Plan generally falls below the standard SAS baseline, with marginal improvements far smaller than those of MAS. This confirms that context isolation, rather than decomposition structure alone, is the operative mechanism behind MAS gains.

\textbf{Does Relay Information Loss Causally Govern Performance?}
Theorem~\ref{theorem:mas-gain-decomposition} predicts that relay information loss $\Delta_i(m_i)$ acts as a direct performance cost. To test this causal mechanism, we construct higher-loss relay variants by deliberately removing downstream-required information while keeping the same MAS decomposition. Specifically, on ALFWorld, we remove location and object-state information from the relay; on WebShop, we retain only the product name while discarding price, attributes, and navigation paths. 

Table~\ref{tab:sufficient_vs_insufficient_relay} shows that increasing relay information loss consistently reduces performance across all models. This provides causal evidence that $\Delta_i(m_i)$ operates as a performance cost rather than a merely descriptive quantity. In particular, when downstream-required information is deliberately omitted from the relay, the performance degradation supports the information-bottleneck prediction that relay compression helps only when it preserves task-relevant information.

% \vspace{-0.8em}
% \begin{table}[!tbp]
% \centering
% \small
% \setlength{\tabcolsep}{5pt}
% \renewcommand{\arraystretch}{0.88}
% \setlength{\abovecaptionskip}{2pt}
% \setlength{\belowcaptionskip}{3pt}
% \caption{MAS performance under sufficient versus insufficient relay on ALFWorld and WebShop.\xj{same format as Table 3?}}
% \label{tab:sufficient_vs_insufficient_relay}
% \vspace{-0.4em}
% \resizebox{0.78\textwidth}{!}{%
% \begin{tabular}{llccc}
% \toprule
% \textbf{Task} & \textbf{Relay} & \textbf{GPT-4o-mini} & \textbf{Qwen2.5-7B-Instruct} & \textbf{Qwen3.5-27B} \\
% \midrule
% ALFWorld 
% & Sufficient   & 0.552 & 0.545 & 0.895 \\
% & Insufficient & 0.537 & 0.530 & 0.843 \\
% & \textbf{Gap} & \cellcolor{red!8}\textbf{-0.015} & \cellcolor{red!8}\textbf{-0.015} & \cellcolor{red!8}\textbf{-0.052} \\
% \midrule
% WebShop
% & Sufficient   & 0.303 & 0.223 & 0.450 \\
% & Insufficient & 0.225 & 0.215 & 0.440 \\
% & \textbf{Gap} & \cellcolor{red!8}\textbf{-0.078} & \cellcolor{red!8}\textbf{-0.008} & \cellcolor{red!8}\textbf{-0.010} \\
% \bottomrule
% \end{tabular}%
% }
% \vspace{-0.8em}
% \end{table}

\begin{table*}[!t]
\centering
\small
\setlength{\tabcolsep}{4pt}
\renewcommand{\arraystretch}{1.10}
\caption{Relay sufficiency ablation on ALFWorld and WebShop. Insufficient relays drop downstream-required information while keeping the same MAS decomposition.}
\label{tab:sufficient_vs_insufficient_relay}
\resizebox{\textwidth}{!}{%
\begin{tabular}{ll ccc ccc ccc}
\toprule
& & \multicolumn{3}{c}{Qwen2.5-7B-Instruct} & \multicolumn{3}{c}{GPT-4o-mini} & \multicolumn{3}{c}{Qwen3.5-27B} \\
\cmidrule(lr){3-5} \cmidrule(lr){6-8} \cmidrule(lr){9-11}
\textbf{Task} & \textbf{Metric} 
& Sufficient & Insufficient & \textbf{Gap} 
& Sufficient & Insufficient & \textbf{Gap} 
& Sufficient & Insufficient & \textbf{Gap} \\
\midrule
ALFWorld & SR 
& 0.545 & 0.530 & \planlower{$-0.015$} 
& 0.552 & 0.537 & \planlower{$-0.015$} 
& 0.895 & 0.843 & \planlower{$-0.052$} \\
\midrule
WebShop & SR 
& 0.223 & 0.215 & \planlower{$-0.008$} 
& 0.303 & 0.225 & \planlower{$-0.078$} 
& 0.450 & 0.440 & \planlower{$-0.010$} \\
\bottomrule
\end{tabular}
}
\end{table*}

\section{Summary and Limitations}
\label{sec:conclusion}
\textbf{Summary.} We presented an information-bottleneck framework that explains when LLM-based multi-agent systems help over single-agent systems. Our key insight is that MAS gains and losses arise from bounded inter-worker relays, which trade off context reduction against downstream-relevant information loss. The effective parameter $\beta$ captures how this trade-off changes with model capability, explaining why the same relay can help weaker models but hurt stronger ones. Across five benchmarks and three model scales, our experiments support this prediction: MAS helps with near-sufficient relays, but loses its advantage when relays become highly insufficient.

% \textbf{Implications.} Our framework turns MAS design into two practical questions. First, when should we use MAS? MAS yields consistent gains whenever a task can be decomposed with no relay information loss, since the only remaining effect is upstream context reduction. Second, in the more common case, how should we design MAS when relay loss is unavoidable? Our information-bottleneck objective points to a natural next step: optimize the relay encoder jointly with the downstream task, for instance through structured output formats trained to preserve downstream-relevant information. The capability-dependent $\beta$ raises a longer-term question: can adaptive relays sustain MAS gains in heterogeneous agent teams with different base models? 
% Such an extension would move our framework from a static decomposition toward a dynamic theory of relay design for heterogeneous and evolving multi-agent systems.

\noindent\textbf{Implications.} Our framework turns MAS design into two practical questions. First, when should we use MAS? When a task admits low relay information loss decomposition, MAS gains follow directly from upstream context reduction. Second, in the more common case, how should we design MAS when relay loss is unavoidable? Our IB objective points to a natural next step: jointly optimizing the relay encoder with the downstream task. A longer-term question raised by the capability-dependent $\beta$ is whether adaptive relays can sustain MAS gains across heterogeneous agent teams with different base models.

\noindent\textbf{Limitations.} Despite these insights, our work has several limitations. We use $\beta$ as a qualitative capability parameter rather than a directly estimated quantity. We assign $\delta$ by task-structure analysis rather than estimating it quantitatively. Future work should develop quantitative estimators for both, and use them to train adaptive relays.

\section*{Acknowledgments}
This work was supported in part by the Advanced Research Projects Agency for Health (ARPA-H) under grant 1AY1AX000053, the Texas A\&M University Division of Research Targeted Proposal Teams Funding Program, and the Texas A\&M Institute of Data Science Thematic Labs Program.

\bibliographystyle{refstyle}
\bibliography{ref}

\newpage
\appendix

\newpage

\section{Theoretical Details}
\subsection{Proof of Proposition~\ref{prop:equivalence}}
\label{app:equivalence-proof}

\begin{proof}
Fix an arbitrary SAS with context trajectory $M_0,M_1,\ldots,M_T$ and final output $y_{\mathrm{SAS}}=\mathrm{Agg}(M_T)$. Let \(0=T_0<T_1<\cdots<T_n=T\) be any partition of the SAS trajectory. We construct an \(n\)-worker MAS that simulates the SAS over the consecutive intervals
\([T_{i-1},T_i]\).

The planner assigns \(X_1=X\) and \(X_i=\varnothing\) for all \(i\ge 2\). Each worker uses the same thought and action policies \((\pi_T,\pi_A)\) as the SAS, and worker \(i\) is responsible for reproducing the SAS updates from time \(T_{i-1}\) to \(T_i\). The relay sent by worker \(i\) is its full terminal context, $m_i = M_{T_i}^{(i)}.$
Equivalently, each worker-level aggregator \(\mathrm{Agg}_i\) is the identity map. The system-level aggregator ignores intermediate worker contexts except for the final one and returns
\[
    \mathrm{Agg}_{\mathrm{MAS}}
    \bigl(M_{T_1}^{(1)},\ldots,M_{T_n}^{(n)}\bigr)
    =
    \mathrm{Agg}\bigl(M_{T_n}^{(n)}\bigr).
\]

We first verify that this construction is feasible under unbounded relay bandwidth. By assumption, the SAS context has finite entropy at every step. Since each relay is exactly a transmitted context \(M_{T_i}^{(i)}\), it also has finite entropy whenever the simulation is well defined. Hence, for any
\[
    B_{\mathrm{sys}}
    \ge
    \max_{1\le i\le n} H\bigl(M_{T_i}^{(i)}\bigr),
\]
all relays satisfy \(H(m_i)\le B_{\mathrm{sys}}\). Thus the construction becomes feasible as \(B_{\mathrm{sys}}\to\infty\).

It remains to prove that the constructed MAS reproduces the SAS trajectory. We show by induction that, for every \(i=1,\ldots,n\), $M_{T_i}^{(i)} \equiv M_{T_i},$
where \(\equiv\) denotes equality of the information contained in the context, up to redundant inclusion of the entry observation.

For the base case, worker \(1\) starts from 
\[
M_0^{(1)}=(X,o_0)=M_0.
\]
Since it uses the same policies \((\pi_T,\pi_A)\) and the same task kernels as the SAS, it generates the same thoughts, actions, observations, and context updates over the interval \([0,T_1]\). Therefore,
\[
    M_{T_1}^{(1)} = M_{T_1}.
\]

For the induction step, assume that $M_{T_i}^{(i)} \equiv M_{T_i}$
for some \(i<n\). Worker \(i+1\) receives the full relay $m_i = M_{T_i}^{(i)}.$
Because \(X_{i+1}=\varnothing\) and the entry observation \(o_0^{(i+1)}=o_{T_i}\) is already contained in the SAS context \(M_{T_i}\), the initial context of worker \(i+1\) satisfies
$
    M_0^{(i+1)}
    =
    (\varnothing,o_0^{(i+1)})\oplus m_i
    \equiv
    M_{T_i}.
$
Starting from an information-equivalent context and using the same policies and task kernels, worker \(i+1\) reproduces the SAS updates over \([T_i,T_{i+1}]\). Hence,
\[
    M_{T_{i+1}}^{(i+1)}
    \equiv
    M_{T_{i+1}}.
\]
This completes the induction.

Taking \(i=n\) gives $M_{T_n}^{(n)} \equiv M_T.$
Therefore, the MAS output is
\[
\begin{aligned}
    y_{\mathrm{MAS}}
    &=
    \mathrm{Agg}_{\mathrm{MAS}}
    \bigl(M_{T_1}^{(1)},\ldots,M_{T_n}^{(n)}\bigr) \\
    &=
    \mathrm{Agg}\bigl(M_{T_n}^{(n)}\bigr) \\
    &=
    \mathrm{Agg}(M_T) \\
    &=
    y_{\mathrm{SAS}}.
\end{aligned}
\]
Thus, for any SAS, there exists a MAS that produces the same final output whenever the relay bandwidth is unbounded. This proves the claim.
\end{proof}
\subsection{Proof of Markov Relation}
\label{app:markov}
\begin{proof}We justify the Markov relation used in Section~\ref{sec:ib_explains}. 
For a fixed planner-induced decomposition, the downstream target $Y_{i+1}$ is determined by the task instance and the local subtask, rather than by the relay compression itself. 
The relay is compressed from the upstream context as
\[
m_i = \mathrm{Agg}_i(M_i),
\]
where $\mathrm{Agg}_i$ may be deterministic or may use external randomness independent of $Y_{i+1}$ given $M_i$ and $X_{i+1}$. 
Therefore, after conditioning on $M_i$ and $X_{i+1}$, observing $m_i$ provides no additional information about $Y_{i+1}$:
\[
p(y_{i+1}\mid m_i, M_i, X_{i+1})
=
p(y_{i+1}\mid M_i, X_{i+1}).
\]
Equivalently,
\[
m_i \perp Y_{i+1} \mid M_i, X_{i+1}.
\]
This conditional independence is exactly the conditional Markov relation
\[
m_i \leftrightarrow M_i \leftrightarrow Y_{i+1} \mid X_{i+1}.
\]
\end{proof}

\subsection{Details on Local-to-Downstream Sufficiency}
\label{app:local-to-downstream}
Given a specific task instance \(X\), a planner produces an ordered
decomposition \(X_{1:n}=(X_1,\ldots,X_n)\). Each sub-instance \(X_i\)
induces a worker-specific target variable \(Y_i\), representing the
task-relevant outcome, state, or decision that worker \(i\) is expected
to infer or produce for its assigned subproblem. We do not assume that
every instance admits a perfect decomposition; rather, whether local
relay sufficiency propagates downstream depends on the structure induced
by the chosen decomposition.

Let \(M_i\) denote the full context of worker \(i\), and let \(m_i\)
denote the relay passed from worker \(i\) to downstream workers. We say
that \(m_i\) is locally sufficient for the next target \(Y_{i+1}\) if
\[
I(Y_{i+1};M_i\mid m_i,X_{i+1})=0.
\]
That is, given \(m_i\) and the next sub-instance \(X_{i+1}\), the full
upstream context \(M_i\) provides no additional information about
\(Y_{i+1}\).

We next ask when this local condition is sufficient for a later target
\(Y_j\), \(j>i+1\). The first required structural condition is mediated
downstream dependence:
\[
I(Y_j;M_i\mid Y_{i+1},X_{i+1},X_j)=0.
\]
This means that, under the induced decomposition, the dependence of
\(Y_j\) on the upstream context \(M_i\) is mediated by the adjacent target
\(Y_{i+1}\) and the relevant sub-instance contexts. Thus, \(Y_{i+1}\)
acts as a local information interface between worker \(i\) and later
workers.

The second condition is downstream-context non-interference:
\[
I(Y_{i+1};M_i\mid m_i,X_{i+1},X_j)
=
I(Y_{i+1};M_i\mid m_i,X_{i+1}).
\]
This rules out cases where a later sub-instance \(X_j\) retroactively
changes the local sufficiency relation between \(M_i\), \(m_i\), and
\(Y_{i+1}\).

Under these two conditions, local sufficiency propagates downstream.
Indeed, local sufficiency and non-interference imply
\[
I(Y_{i+1};M_i\mid m_i,X_{i+1},X_j)=0.
\]
Together with mediated downstream dependence,
\[
I(Y_j;M_i\mid Y_{i+1},X_{i+1},X_j)=0,
\]
we obtain
\[
I(Y_j;M_i\mid m_i,X_{i+1},X_j)=0.
\]
Hence, a relay that is sufficient for the next target is also sufficient
for the later target under this decomposition.

Conversely, if
\[
I(Y_j;M_i\mid Y_{i+1},X_{i+1},X_j)>0,
\]
then the decomposition leaves a hidden long-range dependency from worker
\(i\)'s context to the later target \(Y_j\). In this case, a relay
sufficient for \(Y_{i+1}\) may still be insufficient for \(Y_j\), because
\(M_i\) can contain downstream-relevant information not captured by the
adjacent target. Thus, local-to-downstream sufficiency is not universal;
it depends on whether the task instance admits a decomposition whose
adjacent targets form sufficient information interfaces for later
targets.

\section{Experimental Details}
\subsection{Datasets}
\label{app:benchmarks}

We validate the stage-wise MAS gain across five agentic benchmarks that span a representative range of task structures, relay regimes, and evaluation criteria.

\textbf{ALFWorld}~\cite{shridhar2020alfworld} is a text-based embodied household task benchmark built from ALFRED~\citep{shridhar2020alfred}.
An agent must navigate rooms, locate objects, and complete goals such as \emph{``put a heated egg in the fridge''} through a sequence of symbolic actions.
Tasks decompose naturally into short, spatially local subgoals, such as locating, picking, and placing objects, making upstream navigation history largely irrelevant once the current subgoal state is communicated. We evaluate on the episode \texttt{valid\_unseen} split. We report task success rate (SR) under both zero-shot and few-shot settings.

\textbf{WebShop}~\citep{yao2022webshop} is a simulated e-commerce environment containing over one million products. Each task specifies target attributes, option values, and a price ceiling. The agent interacts with a text browser through search, click, and think actions, and must search products, inspect candidate pages, compare attributes and prices, select required options, and purchase the final item. The final purchase decision requires cross-product attribute comparison and price verification across navigation steps, making the relay regime inherently insufficient. We evaluate on the 100 random selected items. Success requires a perfect reward, meaning that all target attributes, options, and price constraints are satisfied, and average reward measures partial match quality.

\textbf{WorkBench}~\citep{styles2024workbench} is a workplace-software benchmark where agents operate over email, calendar, CRM, project management, and analytics tools to complete natural-language office requests. We use the multi-domain split, which contains 210 tasks spanning two to five domains. About 91\% of queries are conditional and require the agent to first retrieve cross-domain evidence before executing a write action such as \texttt{create\_event}, \texttt{send\_email}, \texttt{create\_task}, or \texttt{add\_customer}. WorkBench is therefore a strongly insufficient-relay setting. The writer depends on exact intermediate read results, including object IDs, timestamps, branch outcomes, and name-to-email resolutions. These details are hard to preserve under bounded relay communication, making relay loss severe. We report task success rate.

% \textbf{WideSearch}~\citep{wong2025widesearch} is a structured information-retrieval benchmark.
% Given a complex query, the agent must retrieve rows of structured evidence across multiple
% independent subqueries.
% We use 66 tasks from the English subset and select tasks that can be decomposed into
% subtasks without requiring cross-subtask relay information.
% Each output row is largely recoverable from its own subquery alone, with negligible
% cross-row coupling, placing this benchmark in the near-sufficient relay regime.
% We report Item~F1 and Row~F1.

\textbf{WideSearch}~\citep{wong2025widesearch} is a structured information-retrieval benchmark. Each task provides a complex query and requires the agent to return a table of evidence. We use 66 tasks from the English subset, where each query can be decomposed along a clear axis, such as entities, categories, or time windows. Each subtask retrieves a separate part of the table, and most rows can be recovered from their own subquery without relying on other rows. WideSearch is therefore a sufficient relay setting because little cross-subtask information must be preserved. We report Item F1 and Row F1.

% \textbf{TravelPlanner}~\cite{xie2024travelplanner} requires agents to generate
% multi-day travel itineraries that simultaneously satisfy transportation, accommodation,
% dining, and attraction constraints.
% Local phase constraints, measured by commonsense checks, are largely relay-sufficient,
% whereas global aggregation constraints, such as total budget and minimum nights, require
% evidence from the full plan history, producing a mixed relay regime.
% We report commonsense micro pass rate~(CS-micro) and hard-constraint micro pass
% rate~(HC-micro) as separate metrics.

\textbf{TravelPlanner}~\citep{xie2024travelplanner} is a trip-planning benchmark requires agents to generate multi-day itineraries under transportation, accommodation, dining, attraction, and budget constraints. We use the 180-query validation split. TravelPlanner has a mixed relay regime. Commonsense checks assess local feasibility and are mostly relay-sufficient, since each planning phase can be validated from compact local information. Hard constraints assess global consistency, including total budget, room type, cuisine coverage, house rules, transportation restrictions, and minimum-night requirements, which require evidence from the full itinerary, making relays insufficient for the global contraints. We report commonsense macro pass rate (CS-macro) and hard-constraint macro pass rate (HC-macro).

\subsection{Models}
\label{app:setting}

We evaluate three base models spanning a capability range:
Qwen2.5-7B-Instruct, GPT-4o-mini, and Qwen3.5-27B-AWQ-4bit.
Qwen2.5-7B-Instruct is served locally with vLLM~\citep{kwon2023efficient}
on a single NVIDIA RTX A6000 GPU in bfloat16. GPT-4o-mini is accessed
through the OpenAI Chat Completions API. Qwen3.5-27B-AWQ-4bit is served
locally with vLLM on a single A6000 GPU using AWQ~\citep{lin2024awq}
4-bit quantization. We use greedy decoding with \texttt{temperature = 0.0}
for all prototypes and benchmarks.

\subsection{Per-benchmark configuration}
Table~\ref{tab:exp-config} summarizes the evaluation settings. For each
benchmark, SAS, SAS-contextflow, and MAS use the same total step budget.
For SAS-contextflow and MAS, the sum of all phase budgets is kept no larger
than the SAS total budget. For WideSearch, all prototypes use DuckDuckGo as
the live web-search backend. All other benchmarks use offline environments.

\begin{table}[h]
\centering
\small
\setlength{\tabcolsep}{5pt}
\renewcommand{\arraystretch}{1.08}
\caption{Per-benchmark experimental settings.}
\label{tab:exp-config}
\begin{tabular}{lcccc}
\toprule
\textbf{Benchmark} & \textbf{Tasks} & \textbf{Total budget} & \textbf{SAS-contextflow / MAS phase steps} & \textbf{Environment} \\
\midrule
ALFWorld      & 134 & 50 & 25 + 25                  & offline simulator \\
WebShop       & 100 & 60 & 30 + 30                  & offline simulator \\
WorkBench     & 210 & 50 & 25 + 25                  & offline sandbox \\
WideSearch    &  66 & 30 & 10 + 10 + 10             & DuckDuckGo (live) \\
TravelPlanner & 180 & 60 & 20 + 12 + 12 + 12 + 4    & offline DB \\
\bottomrule
\end{tabular}
\end{table}

% \begin{table}[h]
% \centering
% \small
% \setlength{\tabcolsep}{5pt}
% \renewcommand{\arraystretch}{1.08}
% \caption{Per-benchmark experimental settings.}
% \label{tab:exp-config}
% \begin{tabular}{lcccc}
% \toprule
% \textbf{Benchmark} & \textbf{Tasks} & \textbf{Max steps} & \textbf{SAS-contextflow / MAS steps} & \textbf{Backend} \\
% \midrule
% ALFWorld      & 134 & 50 & 25 + 25                  & --- \\
% WebShop       & 100 & 60 & 30 + 30                  & --- \\
% WorkBench     & 210 & 50 & 25 + 25                  & --- \\
% WideSearch    &  66 & 30 & 10 + 10 + 10& DuckDuckGo \\
% TravelPlanner & 180 & 60 & 20 + 12 + 12 + 12 + 4      & sandbox DB \\
% \bottomrule
% \end{tabular}
% \end{table}

\section{Additional Results}
\subsection{Detailed Results}
\label{app:results}
Table~\ref{tab:main_results_all_models} presents the main empirical results across all five benchmarks and three model scales, while Tables~\ref{tab:alfworld_taskwise_zeroshot} and~\ref{tab:alfworld_taskwise_fewshot} provide a breakdown of task types for ALFWorld. Tables~\ref{tab:travelplanner_results} and~\ref{tab:widesearch_results} report the detailed metrics for TravelPlanner and WideSearch, respectively.

\begin{table*}[ht]
\centering
\small
\setlength{\tabcolsep}{4pt}
\renewcommand{\arraystretch}{1.08}
\caption{Main results across three models and three prototypes.}
\label{tab:main_results_all_models}
\resizebox{\textwidth}{!}{%
\begin{tabular}{lcccccccccc}
\toprule
\textbf{Prototype}
& \multicolumn{2}{c}{\textbf{ALFWorld}}
& \multicolumn{2}{c}{\textbf{WebShop}}
& \textbf{WorkBench}
& \multicolumn{2}{c}{\textbf{WideSearch}}
& \multicolumn{3}{c}{\textbf{TravelPlanner}} \\
\cmidrule(lr){2-3}\cmidrule(lr){4-5}\cmidrule(lr){6-6}\cmidrule(lr){7-8}\cmidrule(lr){9-11}
& \textbf{Zero-shot}
& \textbf{Few-shot}
& \textbf{SR}
& \textbf{Avg Reward}
& \textbf{SR}
& \textbf{Item F1}
& \textbf{Row F1}
& \textbf{SR}
& \textbf{CS macro}
& \textbf{HC macro} \\
\midrule

\rowcolor{blue!10}
\multicolumn{11}{c}{\rule{0pt}{2.4ex}\textbf{Qwen2.5-7B-Instruct}} \\
SAS             & 0.366 & 0.679 & 0.173 & 0.557 & 0.205 & 0.208 & 0.015 & 0.000 & 0.006 & 0.000 \\
SAS-contextflow & 0.172 & 0.485 & 0.143 & 0.557 & 0.176 & 0.244 & 0.039 & 0.000 & 0.006 & 0.000 \\
MAS             & 0.545 & 0.679 & 0.223 & 0.579 & 0.171 & 0.322 & 0.063 & 0.000 & 0.017 & 0.017 \\
\addlinespace[3pt]

\rowcolor{blue!10}
\multicolumn{11}{c}{\rule{0pt}{2.4ex}\textbf{GPT-4o-mini}} \\
SAS             & 0.507 & 0.851 & 0.353 & 0.666 & 0.490 & 0.387 & 0.109 & 0.011 & 0.128 & 0.056 \\
SAS-contextflow & 0.373 & 0.679 & 0.217 & 0.432 & 0.514 & 0.465 & 0.119 & 0.006 & 0.044 & 0.011 \\
MAS             & 0.552 & 0.836 & 0.303 & 0.654 & 0.428 & 0.528 & 0.159 & 0.028 & 0.228 & 0.172 \\
\addlinespace[3pt]

\rowcolor{blue!10}
\multicolumn{11}{c}{\rule{0pt}{2.4ex}\textbf{Qwen3.5-27B}} \\
SAS             & 0.903 & 0.993 & 0.447 & 0.725 & 0.639 & 0.630 & 0.345 & 0.089 & 0.150 & 0.433 \\
SAS-contextflow & 0.888 & 0.910 & 0.453 & 0.721 & 0.619 & 0.628 & 0.304 & 0.117 & 0.167 & 0.489 \\
MAS             & 0.895 & 0.933 & 0.450 & 0.733 & 0.605 & 0.656 & 0.338 & 0.056 & 0.194 & 0.256 \\

\bottomrule
\end{tabular}
}
\end{table*}

\begin{table*}[!htbp]
\centering
\small
\setlength{\tabcolsep}{3pt}
\caption{ALFWorld task-wise success rates in the zero-shot setting.}
\label{tab:alfworld_taskwise_zeroshot}
\resizebox{\textwidth}{!}{%
\begin{tabular}{lccccccccc}
\toprule
& \multicolumn{3}{c}{\textbf{Qwen2.5-7B-Instruct}}
& \multicolumn{3}{c}{\textbf{GPT-4o-mini}}
& \multicolumn{3}{c}{\textbf{Qwen3.5-27B}} \\
\cmidrule(lr){2-4}\cmidrule(lr){5-7}\cmidrule(lr){8-10}
\textbf{Task Type}
& \textbf{SAS}
& \textbf{SAS-contextflow}
& \textbf{MAS}
& \textbf{SAS}
& \textbf{SAS-contextflow}
& \textbf{MAS}
& \textbf{SAS}
& \textbf{SAS-contextflow}
& \textbf{MAS} \\
\midrule
pick\_and\_place         & 0.833 & 0.375 & 0.750 & 0.750 & 0.625 & 0.833 & 1.000 & 0.958 & 0.917 \\
pick\_clean\_then\_place & 0.387 & 0.065 & 0.645 & 0.387 & 0.129 & 0.323 & 0.871 & 0.968 & 0.871 \\
pick\_cool\_then\_place  & 0.286 & 0.143 & 0.714 & 0.333 & 0.286 & 0.524 & 0.952 & 0.952 & 0.905 \\
pick\_heat\_then\_place  & 0.130 & 0.261 & 0.435 & 0.217 & 0.087 & 0.565 & 0.652 & 0.478 & 0.783 \\
pick\_two\_obj           & 0.176 & 0.059 & 0.235 & 0.765 & 0.706 & 0.471 & 1.000 & 1.000 & 0.941 \\
look\_at\_obj            & 0.278 & 0.111 & 0.333 & 0.722 & 0.611 & 0.667 & 1.000 & 1.000 & 1.000 \\
\bottomrule
\end{tabular}
}
\end{table*}

\begin{table*}[!htbp]
\centering
\small
\setlength{\tabcolsep}{3pt}
\caption{ALFWorld task-wise success rates in the few-shot setting.}
\label{tab:alfworld_taskwise_fewshot}
\resizebox{\textwidth}{!}{%
\begin{tabular}{lccccccccc}
\toprule
& \multicolumn{3}{c}{\textbf{Qwen2.5-7B-Instruct}}
& \multicolumn{3}{c}{\textbf{GPT-4o-mini}}
& \multicolumn{3}{c}{\textbf{Qwen3.5-27B}} \\
\cmidrule(lr){2-4}\cmidrule(lr){5-7}\cmidrule(lr){8-10}
\textbf{Task Type}
& \textbf{SAS}
& \textbf{SAS-contextflow}
& \textbf{MAS}
& \textbf{SAS}
& \textbf{SAS-contextflow}
& \textbf{MAS}
& \textbf{SAS}
& \textbf{SAS-contextflow}
& \textbf{MAS} \\
\midrule
pick\_and\_place         & 0.917 & 0.792 & 0.625 & 0.917 & 0.792 & 0.875 & 1.000 & 0.958 & 1.000 \\
pick\_clean\_then\_place & 0.742 & 0.419 & 0.645 & 0.871 & 0.387 & 1.000 & 0.968 & 1.000 & 1.000 \\
pick\_cool\_then\_place  & 0.810 & 0.476 & 0.714 & 0.857 & 0.667 & 0.762 & 1.000 & 0.905 & 0.952 \\
pick\_heat\_then\_place  & 0.826 & 0.652 & 0.913 & 0.826 & 0.870 & 0.826 & 1.000 & 0.609 & 0.696 \\
pick\_two\_obj           & 0.176 & 0.235 & 0.412 & 0.706 & 0.706 & 0.588 & 1.000 & 1.000 & 0.941 \\
look\_at\_obj            & 0.389 & 0.222 & 0.722 & 0.889 & 0.778 & 0.833 & 1.000 & 1.000 & 1.000 \\
\bottomrule
\end{tabular}
}
\end{table*}

\begin{table*}[!htbp]
\centering
\small
\setlength{\tabcolsep}{4pt}
\caption{TravelPlanner results across three prototypes and three models.}
\label{tab:travelplanner_results}
\resizebox{\textwidth}{!}{%
\begin{tabular}{lccccccccc}
\toprule
& \multicolumn{3}{c}{\textbf{Qwen2.5-7B-Instruct}}
& \multicolumn{3}{c}{\textbf{GPT-4o-mini}}
& \multicolumn{3}{c}{\textbf{Qwen3.5-27B}} \\
\cmidrule(lr){2-4}\cmidrule(lr){5-7}\cmidrule(lr){8-10}
\textbf{Metric}
& \textbf{SAS}
& \textbf{SAS-contextflow}
& \textbf{MAS}
& \textbf{SAS}
& \textbf{SAS-contextflow}
& \textbf{MAS}
& \textbf{SAS}
& \textbf{SAS-contextflow}
& \textbf{MAS} \\
\midrule
Final Pass Rate         & 0.000 & 0.000 & 0.000 & 0.011 & 0.006 & 0.028 & 0.089 & 0.117 & 0.056 \\
Commonsense Macro Pass  & 0.006 & 0.006 & 0.017 & 0.128 & 0.044 & 0.228 & 0.150 & 0.167 & 0.194 \\
Hard Macro Pass Rate    & 0.000 & 0.000 & 0.017 & 0.056 & 0.011 & 0.172 & 0.433 & 0.489 & 0.256 \\
Delivery Rate           & 0.922 & 0.756 & 1.000 & 0.928 & 0.872 & 1.000 & 0.989 & 0.983 & 0.967 \\
Commonsense Micro Pass  & 0.509 & 0.374 & 0.587 & 0.700 & 0.642 & 0.820 & 0.750 & 0.767 & 0.800 \\
Hard Micro Pass Rate    & 0.043 & 0.019 & 0.029 & 0.095 & 0.021 & 0.262 & 0.433 & 0.450 & 0.286 \\
\bottomrule
\end{tabular}
}
\end{table*}

\begin{table*}[!htbp]
\centering
\small
\setlength{\tabcolsep}{3pt}
\caption{WideSearch results across three prototypes and three models.}
\label{tab:widesearch_results}
\resizebox{\textwidth}{!}{%
\begin{tabular}{lccccccccc}
\toprule
& \multicolumn{3}{c}{\textbf{Qwen2.5-7B-Instruct}}
& \multicolumn{3}{c}{\textbf{GPT-4o-mini}}
& \multicolumn{3}{c}{\textbf{Qwen3.5-27B}} \\
\cmidrule(lr){2-4}\cmidrule(lr){5-7}\cmidrule(lr){8-10}
\textbf{Metric}
& \textbf{SAS}
& \textbf{SAS-contextflow}
& \textbf{MAS}
& \textbf{SAS}
& \textbf{SAS-contextflow}
& \textbf{MAS}
& \textbf{SAS}
& \textbf{SAS-contextflow}
& \textbf{MAS} \\
\midrule
Score (\%)          & 0.00 & 0.00 & 0.00 & 1.52 & 0.00 & 0.00 & 4.55 & 3.03 & 0.00 \\
Item F1 (\%)        & 20.83 & 24.35 & 32.24 & 38.69 & 46.52 & 52.77 & 62.96 & 62.76 & 65.60 \\
Row F1 (\%)         & 1.47 & 3.94 & 6.26 & 10.93 & 11.85 & 15.88 & 34.53 & 30.43 & 33.76 \\
Item Precision (\%) & 35.79 & 30.33 & 39.95 & 50.27 & 56.61 & 60.38 & 68.41 & 69.53 & 72.39 \\
Row Precision (\%)  & 2.17 & 4.72 & 8.33 & 14.19 & 13.45 & 18.26 & 37.08 & 33.91 & 35.83 \\
Item Recall (\%)    & 17.82 & 22.92 & 29.95 & 35.18 & 43.41 & 50.50 & 61.05 & 60.53 & 63.39 \\
Row Recall (\%)     & 1.38 & 3.62 & 5.53 & 9.76 & 11.26 & 14.94 & 33.41 & 29.36 & 32.88 \\
\bottomrule
\end{tabular}
}
\end{table*}
\FloatBarrier

\subsection{Case Study: MAS Gain Flip across $\beta$ in TravelPlanner}
\label{app:case_study}
We use TravelPlanner task \texttt{tp\_1} to show the capability-dependent MAS gain reversal. The task asks for a three-day itinerary from Oakland to Tucson under a hard budget of \$1{,}400, with \texttt{budget} kept as a plan-level global and hidden from all MAS sub-agents. On GPT-4o-mini, MAS produces the only HC-valid plan, while SAS-contextflow leaves Day-1 and Day-2 accommodation slots empty and is gated by \texttt{is\_not\_absent} before \texttt{valid\_cost} is evaluated. On Qwen3.5-27B, SAS-contextflow can use the visible \$1{,}400 cap in its accumulated context and stays within budget, while MAS loses this signal through bounded relays and produces a \$1{,}599 itinerary. Thus, \(\beta\Delta_i\) dominates at higher capability and the MAS gain flips. Together, these two regimes show the predicted transition: as model capability increases, MAS gains from context reduction become smaller, while the cost of relay information loss becomes more visible.

\begin{promptbox}[Case: TravelPlanner \texttt{tp\_1} MAS Gain flip across $\beta$]
Query: "Please draw up a 3-day travel itinerary for one person, beginning in
Oakland and heading to Tucson from March 15th to March 17th, 2022, with a
budget of $1,400."

Code-defined sub-specs (build_initial_subqueries):
  transport     = {org:"Oakland", dest:"Tucson", days:3, people_number:1, ...}
  accommodation = {people_number:1, house_rule:null, room_type:null, city_list:["Tucson"]}
  dining        = {org:"Oakland", cuisine:null, city_list:["Tucson"]}
  attraction    = {city_list:["Tucson"]}
Plan-level global (NOT in any sub-spec):
  budget = $1,400

==================================================================================
GPT-4o-mini (mid beta)        SAS-contextflow                   MAS
----------------------------------------------------------------------------------
Day 1 from Oakland to Tucson  Flight F4002752                   Flight F4002752
      Breakfast               Malais By Anands, Oakland         Malais By Anands, Oakland
      Lunch                   Villa Tevere, Tucson              Villa Tevere, Tucson
      Dinner                  Mocha, Tucson                     La Plage, Tucson
      Attraction              Pima Air & Space Museum           Pima Air & Space Museum
      Accommodation           -                    [missing]    Private room with private bathroom

Day 2 Tucson                  -                                 -
      Breakfast               Chai Point, Tucson                Mocha, Tucson
      Lunch                   Ooh Lala, Tucson                  Pirates of Grill, Tucson
      Dinner                  Pirates of Grill, Tucson          Chutneez Restaurant Lounge & Bar
      Attraction              Reid Park Zoo + Tucson Botanical  Reid Park Zoo + Tucson Botanical
      Accommodation           -                    [missing]    Private room with private bathroom

Day 3 from Tucson to Oakland  Self-driving                      Self-driving
      Breakfast               Ramu Ka Dhaba, Tucson             Giani's, Tucson
      Lunch                   Giani's, Tucson                   Domino's Pizza, Tucson
      Dinner                  Chef Sack Kitchen, Oakland        -
      Attraction              Arizona-Sonora Desert Museum      The Mini Time Machine Museum
      Accommodation           -                                 -

Summary                       
                              HC=Fail (CS-gated)                HC=PASS  (cost <= $1,400)
                                             

==================================================================================
Qwen3.5-27B (high beta)       SAS-contextflow                   MAS
----------------------------------------------------------------------------------
Day 1 from Oakland to Tucson  Flight F4002752                   Flight F4002752
      Breakfast               Cho Gao - Crowne Plaza, Oakland   Cho Gao - Crowne Plaza, Oakland
      Lunch                   Moonshine Cafe & Bar, Oakland     Villa Tevere, Tucson
      Dinner                  Villa Tevere, Tucson              La Plage, Tucson
      Attraction              Pima Air & Space Museum           Pima Air & Space Museum
      Accommodation           Private room with private         Private and Quiet Room
                              bathroom                          in the Perfect Location

Day 2 Tucson                  -                                 -
      Breakfast               Mocha, Tucson                     Mocha, Tucson
      Lunch                   Pizza Street, Tucson              Pirates of Grill, Tucson
      Dinner                  La Plage, Tucson                  Culture Club - Bar De Tapas
      Attraction              Arizona-Sonora Desert Museum      Arizona-Sonora Desert Museum
                                                                + Reid Park Zoo
      Accommodation           Private room with private         Private and Quiet Room
                              bathroom                          in the Perfect Location

Day 3 from Tucson to Oakland  Self-driving                      Self-driving
      Breakfast               Magic Spice Wok, Tucson           Uraki, Tucson
      Lunch                   Consort Restaurant, Tucson        Ooh Lala, Tucson
      Dinner                  Pind Balluchi, Oakland            Moonshine Cafe & Bar, Oakland
      Attraction              Reid Park Zoo                     Tucson Botanical Gardens
      Accommodation           -                                 -

Summary                       
                              HC=PASS  (cost <= $1,400)         HC=Fail Total cost exceeds budget
\end{promptbox}

\subsection{Case Study: Relay Information Loss Dominates Across $\beta$ in WorkBench}
We use WorkBench task \texttt{multi\_domain\_165} to show relay information loss affect across model capabiltiy. The task asks the agent to check whether \texttt{total\_visits} exceeded 10 since September 17 and, if so, create a backlog task for kofi with a deadline of ``next Friday''. Since today is Thursday, 2023-11-30, the correct deadline is 2023-12-08. SAS-contextflow succeeds on both models because the analytics result, the phrase ``next Friday'', the current date, and the write step remain in one accumulated context. MAS fails on both models through different relay-loss mechanisms. On GPT-4o-mini, the reader derives the correct date \texttt{2023-12-08}, but its attempted \texttt{create\_task} call is blocked by the reader tool whitelist, and the relay collapses to \texttt{RESULT: NO\_ACTION, FIELDS:\{\}}, dropping the correct conclusion. On Qwen3.5-27B, the reader resolves ``next Friday'' to \texttt{2023-12-01} and locks the wrong date into \texttt{FIELDS}, which the writer then executes. In both cases, strong relay information loss happens on the deadline information needed by the writer. The relay carries only concrete write-tool fields, while SAS-contextflow keeps the evidence needed to resolve the deadline at execution time. This case show how strong relay information loss affects performance across model capabilities.

\begin{promptbox}[Case: WorkBench \texttt{multi\_domain\_165} relay-compression loss]
Query: "Was total visits more than 10 at any time since September 17? If so,
make a backlog task called 'Improve total visits' for kofi on the front-end
board with a deadline of next Friday otherwise send them an email titled
'Recent total visits' saying 'I noticed total visits has been stable, nice
work!'"

Sandbox state (today = Thursday, 2023-11-30):
  analytics.total_visits_count(2023-09-17 .. 2023-11-30)
    -> {'2023-09-17': 10, '2023-09-18': 12, '2023-09-19': 21, ...}    >10 hit
  "next Friday" relative to today (Thu 2023-11-30) = 2023-12-08

Ground-truth action:
  project_management.create_task(task_name="Improve total visits",
    assigned_to_email="kofi.mensah@atlas.com", list_name="Backlog",
    due_date="2023-12-08", board="Front end")

==================================================================================
GPT-4o-mini (mid beta)        SAS-contextflow                  MAS
----------------------------------------------------------------------------------
Read step                     total_visits_count(...)          total_visits_count(...)
                                -> {...,'2023-09-19': 21,...}    -> {...,'2023-09-19': 21,...}
                              "exceeded 10 on multiple days,   "exceeded 10 on multiple
                               next Friday = December 8,        days, next Friday =
                               2023"                             December 8, 2023"
                              create_task(due_date=             create_task(due_date=
                                "2023-12-08", ...)                "2023-12-08", ...)
                                -> assignee not valid             -> tool not allowed
                              find_email_address("Kofi")          in reader phase
                                -> kofi.mensah@atlas.com        find_email_address("kofi")
                              create_task(due_date=               -> kofi.mensah@atlas.com
                                "2023-12-08",                   send_email(...)
                                assigned_to_email=                -> tool not allowed
                                "kofi.mensah@atlas.com")        "I cannot create a task
                                                                 or send an email, so I
                                                                 will report no action."

Relay (reader->writer)        n/a (single context)             RESULT: NO_ACTION
                                                                FIELDS: {}
                                                              # M_i had the right date
                                                              # but the reader misread
                                                              # its tool-permission errors
                                                              # as "no action possible"

Write step                    create_task succeeds             writer reads NO_ACTION
                                                                -> Subtask Complete
                                                                   (no write performed)

Summary                       HC=PASS                          HC=FAIL
                              create_task with                 reader's correct conclusion
                              due_date=2023-12-08               was discarded at the
                                                                 M_i -> m_i compression

==================================================================================
Qwen3.5-27B (high beta)       SAS-contextflow                  MAS
----------------------------------------------------------------------------------
Read step                     total_visits_count(...)          total_visits_count(...)
                                -> {...'2023-09-19': 21,...}     -> {...'2023-09-19': 21,...}
                              "condition met"                  "condition met"
                              find_email_address("kofi")       find_email_address("kofi")
                                -> kofi.mensah@atlas.com         -> kofi.mensah@atlas.com
                              "deadline of next Friday         "Today is Thursday,
                               (2023-12-08)"                     2023-11-30. Next Friday
                                                                 would be 2023-12-01."

Relay (reader->writer)        n/a (single context)             RESULT: ACTION_NEEDED
                                                                FIELDS:
                                                                  task_name:
                                                                    "Improve total visits"
                                                                  assigned_to_email:
                                                                    "kofi.mensah@atlas.com"
                                                                  list_name: "Backlog"
                                                                  due_date: "2023-12-01"
                                                                  board: "Front end"
                                                                            [wrong date]
                                                              # FIELDS has no slot for
                                                              # the underlying phrase
                                                              # "next Friday" or for
                                                              # today's date

Write step                    create_task(                     writer reads FIELDS,
                                due_date="2023-12-08",          obeys verbatim per the
                                ...,                            "trust the relay" rule
                                assigned_to_email=             create_task(
                                "kofi.mensah@atlas.com",         due_date="2023-12-01",
                                board="Front end")               ...)

Summary                       HC=PASS                          HC=FAIL
                              create_task with                 reader compressed
                              due_date=2023-12-08               "next Friday" -> wrong
                                                                 concrete date; writer
                                                                 cannot recover
\end{promptbox}

\section{Prompts}
\label{app:prompts}

Here, we provide the planner-induced decompositions for all five agentic tasks in Appendix~\ref{app:decomposition}. We provide the prompts used by SAS, SAS-contextflow, and MAS in Appendices~\ref{app:sas_prompts} and~\ref{app:contextflow_prompts}, and the prompts used by SAS-plan in Appendix~\ref{app:sasplan_prompt}.

\subsection{Task Decomposition and Relay Design}
\label{app:decomposition}
\paragraph{ALFWorld.}
We use a sequential two-phase decomposition, \emph{Explorer} $\rightarrow$ \emph{Executor}. 
The explorer performs only navigation and inspection actions, such as \texttt{go to}, \texttt{open}, \texttt{examine}, \texttt{look}, and \texttt{inventory}. 
Manipulation actions are disabled. 
The explorer passes a structured relay containing each target object and its location. 
The executor reads this relay, skips exploration, and performs the required manipulation actions. 
Thus, the relay only needs to preserve object identity and location, making ALFWorld a near-sufficient relay setting.

\begin{promptbox}[ALFWorld relay example]
<report>
Target object 1: egg 2, found at: countertop 1
</report>
\end{promptbox}

% \paragraph{WebShop.}
% We use a sequential two-phase decomposition, \emph{Searcher} $\rightarrow$ \emph{Verifier}. The searcher browses and filters products according to the user instruction, then emits a relay summarizing the candidate products and their key attributes. The verifier reads this relay and selects the final product to purchase.

\paragraph{WebShop.}
We use a sequential two-phase decomposition, \emph{Searcher} $\rightarrow$ \emph{Verifier}. 
The searcher browses and filters products according to the purchase instruction. 
It passes a relay summarizing candidate products, prices, available options. 
The verifier reads this relay and selects the final product to purchase.

\begin{promptbox}[WebShop relay example]
<report>
Product 1: CeraVe Moisturizing Cream, ASIN: B00TTD9BRC, Price: \$14.99, Options: [50ml, 100ml]
Product 2: Cetaphil Daily Hydrating Lotion, ASIN: B07XYZ123, Price: \$24.99, Options: [50ml, fragrance-free]
Best match: B07XYZ123 because matches all required attributes (50ml, fragrance-free, sensitive skin) and is under \$30
</report>
\end{promptbox}

\paragraph{WorkBench.}
We use a sequential two-phase decomposition, \emph{Reader} $\rightarrow$ \emph{Writer}. The reader has all search, lookup, and retrieval tools across calendar, email, CRM, project management, analytics, and the company directory. It then sends a relay with \texttt{RESULT} and \texttt{FIELDS}. \texttt{RESULT} is either \texttt{ACTION\_NEEDED} or \texttt{NO\_ACTION}, and \texttt{FIELDS} contains the exact parameters for the writer tool when an action is needed. The writer has only writing tools and executes the action from the provided \texttt{FIELDS}, or halts when it receives \texttt{NO\_ACTION}. WorkBench is strongly relay-insufficient because correct execution depends on exact read-stage evidence, including record IDs, timestamps, branch outcomes, and name-to-email mappings.

\begin{promptbox}[WorkBench relay example]
RESULT: ACTION_NEEDED
FIELDS:
  branch:                "A"
  task_name:             "Improve average session duration"
  board:                 "Front end"
  list_name:             "Backlog"
  assigned_to_email:     "chenwei.zhang@atlas.com"
  due_date:              "2023-12-08"
  event_name:            "Discuss total visits"
  participant_email:     "chenwei.zhang@atlas.com"
  event_start:           "2023-12-01 13:00:00"
  duration:              "30"
\end{promptbox}

\paragraph{WideSearch.}
WideSearch is highly decomposable since its subtasks require little inter-worker information exchange. We therefore fix the subtask decomposition for each query and perform no relay among workers. For each WideSearch task, we prompt gpt-5 to precompute a fixed decomposition into exactly 3 subtasks. The split follows the natural enumeration axis of the task, such as entities, categories, or time windows. Each worker answers one subtask using the original table schema, and the final answer is formed by concatenating the returned table fragments. Therefore, wideSearch is near-sufficient because each row is usually determined by its own subtask, with little dependence on other subtasks.

\begin{promptbox}[WideSearch decomposition example]
Original task:
"I'm currently mapping the product portfolios of several spirits brands including Johnnie Walker, Chivas Regal, Smirnoff, Grey Goose, Absolut Vodka, Bacardi as of June 2025."

3 subtasks:
[0] Research ONLY the Core / Permanent Range standard products for these spirits brands: Johnnie Walker, Smirnoff, as of June 2025. [schema unchanged]

[1] Research ONLY the Core / Permanent Range standard products for these spirits brands: Chivas Regal, Absolut Vodka, as of June 2025. 

[2] Research ONLY the Core / Permanent Range standard products for these spirits brands: Grey Goose, Bacardi, as of June 2025. 
\end{promptbox}
\begin{promptbox}[Workbench gpt-5 decompostion prompt]
You are an orchestrator for parallel web research. Two phases: decompose, synthesize.

You DO NOT know the answers. Sub-agents find them by searching the web. Your job is to write RESEARCH INSTRUCTIONS, never data.

## Decompose — output EXACTLY 3 MECE subtasks (100% coverage, zero overlap)

1. Enumerate units. List every atomic item the task names — each brand/person/country/event/year. Expand "including X, Y, Z" and "from A to B" into explicit lists. Call the count K.
2. Pick axis (priority: entity > region > period > attribute). An axis is valid only if each bucket hits a DIFFERENT primary source. If two buckets would open the same main page, pick another axis.
3. Partition into 3 buckets:
   - K > 3: ceil(K/3) or floor(K/3) units per bucket; group units that share a source (same parent company, same country). Drop nothing.
   - K = 2 or 3: one unit per bucket; if K = 2, split the source-densest unit along a sub-axis for the 3rd bucket.
   - K = 1: split by sub-scopes the task names (event levels, time segments, categories) into 3 buckets.
4. Each subtask "prompt" field MUST be an imperative research instruction beginning with "Research...", "Find...", or "Compile...". It must (a) explicitly NAME the assigned units, (b) list the required columns/fields, (c) copy all global constraints (dates, geography, formatting, exclusions, "do not ask questions"). NEVER write answer values (dates, specs, numbers, rankings). You do not know them. If you catch yourself typing "Release Date: 2023-...", stop and rewrite as an instruction.

Before output, verify: every enumerated unit assigned exactly once; exactly 3 buckets (indexes 0, 1, 2); each "prompt" starts with an imperative verb and contains no invented values. Return ONLY the JSON array.
\end{promptbox}

\paragraph{TravelPlanner.}
TravelPlanner tasks follow a fixed input schema with origin, destination, dates, group size, budget, and local constraints on transportation, accommodation, and cuisine. We therefore use deterministic code-based decomposition. Each query is split into four role-specific JSON sub-specs: \texttt{transport}, \texttt{accommodation}, \texttt{dining}, and \texttt{attraction}. Each JSON sub-spec contains only the fields and constraints needed by that role. Transportation constraints are routed to \texttt{transport}, house rules and room type to \texttt{accommodation}, and cuisine preferences to \texttt{dining}. The budget is kept as a global plan-level constraint for final synthesis. The only cross-role relay is \texttt{city\_list}: \texttt{transport\_searcher} first selects feasible cities, and this list is injected into the other three JSON sub-specs before they run in parallel. TravelPlanner is a mixed-relay setting because commonsense checks are mostly local to each role, while hard checks, especially total budget, require evidence from the full itinerary.

\begin{promptbox}[TravelPlanner decomposition example]
Original query:
We're planning a week-long trip for two from Pittsburgh to New York with a budget of \$5,300.
We're set to travel from March 13th to March 19th, 2022, and plan to visit three different
cities in New York. Please keep in mind that our lodgings must allow visitors. As for meals,
we'd love to sample French, Italian, Chinese, and American cuisines. Also, note that we're
planning to travel without taking any flights.

Code-defined sub-specs (after transport_searcher resolves city_list):

transport     = {"org":"Pittsburgh", "dest":"New York", "dest_type":"state",
                 "visiting_city_number":3, "days":7,
                 "date":["2022-03-13", ..., "2022-03-19"],
                 "people_number":2, "mode_blacklist":["flight"]}

accommodation = {"people_number":2, "house_rule":"visitors", "room_type":null,
                 "city_list":["Buffalo","Rochester","Syracuse"]}

dining        = {"org":"Pittsburgh",
                 "cuisine":["French","Italian","Chinese","American"],
                 "city_list":["Buffalo","Rochester","Syracuse"]}

attraction    = {"city_list":["Buffalo","Rochester","Syracuse"]}
\end{promptbox}

\subsection{SAS Prompts}
\label{app:sas_prompts}
\begin{promptbox}[ALFWorld]
You are a robot agent solving household tasks in a text-based environment. Think step-by-step before each action. Use past experience from context when available.

CRITICAL ACTION RULES:
- To COOL an object: use "cool <obj> with <recep>" (e.g. cool lettuce 1 with fridge 1). Do NOT use "move".
- To HEAT an object: use "heat <obj> with <recep>" (e.g. heat egg 1 with microwave 1). Do NOT use "move".
- To CLEAN an object: use "clean <obj> with <recep>" (e.g. clean plate 1 with sinkbasin 1). Do NOT use "move".
- "move <obj> to <recep>" only places an object at a location — it does NOT change its temperature or cleanliness.
- After cooling/heating/cleaning, you still need to "move" the object to the final destination.
- If you think a task is complete, verify by checking inventory and task description before stopping.

You are now in a household environment. Your task is to find, manipulate, and place objects.

CRITICAL: Every object and receptacle has a NUMBER in its name. You MUST always include this number. CORRECT: take saltshaker 1 from countertop 3 WRONG: take saltshaker from countertop 3 The number comes from what you see in observations (e.g. "a saltshaker 1", "a drawer 2").

Valid actions (use EXACTLY these formats):
go to <receptacle N>
open <receptacle N>
close <receptacle N>
examine <receptacle N>
take <object N> from <receptacle N>
move <object N> to <receptacle N> -- place an object in or on a receptacle
clean <object N> with <receptacle N>
heat <object N> with <receptacle N>
cool <object N> with <receptacle N>
use <object N>
inventory
look
think: <reasoning>

Rules:
1. NAVIGATION: Always "go to <receptacle N>" before interacting with it. "Nothing happens." means you are not there yet.
2. CARRY LIMIT: You can hold only ONE object at a time.
3. PLACING: Use "move <object N> to <receptacle N>" to place objects.
4. look_at_obj tasks: take the object, go to the desklamp, then use the desklamp.
5. Never loop on "think:" — always attempt the next concrete action.
\end{promptbox}
\begin{promptbox}[Webshop]
You are a shopping agent navigating WebShop to find and purchase a product that exactly matches the given instruction.
The ONLY valid actions are:
  search[query]   — search for products; only valid on the home page
  click[button]   — click a button, product ASIN, option value, or sub-page name
Navigation flow:
  Home → search[...] → Search results → click[ASIN] → Item page
       → click[option] (repeat for each required option)
       → click[Buy Now] → Done
Rules:
  - Price must be strictly below the stated limit.
  - You MUST select ALL required options (color, size, etc.) before clicking Buy Now.
  - Your total step budget is 60. You MUST purchase a product before your step budget runs out. If you are near the final steps and have not found a perfect match, buy the closest available product immediately.
  - If you receive "Invalid action format", you used a wrong verb — use click[...] or search[...] immediately.
  - Output exactly one action per turn.
  - You may ONLY click ASINs (e.g. click[B01ABCDEF]) that appear in the CURRENT page.
\end{promptbox}
\begin{promptbox}[Workbench]
Today's date is Thursday, 2023-11-30 and the current time is 00:00:00. Remember the current date and time when answering queries. Meetings must not start before 9am or end after 6pm.
Available tools:
[calendar.get_event_information_by_id]
calendar.get_event_information_by_id(event_id="XXXXXXXX", field="<field>")
  Returns one field of a calendar event.
  field must be one of: "event_id", "event_name", "participant_email", "event_start", "duration".
  duration is in minutes.
[calendar.search_events]
calendar.search_events(query="<text>", time_min="YYYY-MM-DD HH:MM:SS", time_max="YYYY-MM-DD HH:MM:SS")
  Search calendar events. query matches event_name and participant_email.
  time_min / time_max filter by event_start (inclusive). All parameters optional.
  Returns at most 5 events. Use query='' to list all events in the window.
  You MUST call this tool TWICE with non-overlapping half-day windows:
    call 1: time_min='YYYY-MM-DD 09:00:00', time_max='YYYY-MM-DD 13:00:00'
    call 2: time_min='YYYY-MM-DD 13:00:00', time_max='YYYY-MM-DD 18:00:00'
  Combine both results before determining free slots.
  To find the first free slot: for each event compute end_time = event_start + duration minutes.
  A slot at time T is occupied if any event satisfies event_start <= T < end_time.
  Find the earliest T (on the hour or half-hour, >= 09:00) where the desired window is fully unoccupied.
[calendar.create_event]
calendar.create_event(event_name="...", participant_email="...", event_start="YYYY-MM-DD HH:MM:SS", duration="<minutes>")
  Create a new calendar event. All parameters required.
  Returns the new 8-digit event_id.
[calendar.delete_event]
calendar.delete_event(event_id="XXXXXXXX")
  Delete a calendar event by its 8-digit ID.
[calendar.update_event]
calendar.update_event(event_id="XXXXXXXX", field="<field>", new_value="...")
  Update one field of an existing calendar event. All parameters required.
  field must be one of: "event_name", "participant_email", "event_start", "duration".
[email.get_email_information_by_id]
email.get_email_information_by_id(email_id="<id>", field="<field>")
  Returns one field of an email.
  field must be one of: "email_id", "inbox/outbox", "sender/recipient", "subject", "sent_datetime", "body".
[email.search_emails]
email.search_emails(query="<text>", date_min="YYYY-MM-DD", date_max="YYYY-MM-DD")
  Search emails. All words in query must appear in subject, body, or sender/recipient (AND logic).
  date_min / date_max filter by sent date (inclusive). All parameters optional.
  Returns at most 5 emails sorted by sent_datetime descending.
[email.send_email]
email.send_email(recipient="<email>", subject="...", body="...")
  Send a new email. All parameters required.
[email.delete_email]
email.delete_email(email_id="<id>")
  Delete an email by its ID.
[email.forward_email]
email.forward_email(email_id="<id>", recipient="<email>")
  Forward an email to a new recipient. Subject is prefixed with "FW: ".
[email.reply_email]
email.reply_email(email_id="<id>", body="...")
  Reply to an email using the original sender as recipient and the same subject.
[analytics.get_visitor_information_by_id]
analytics.get_visitor_information_by_id(visitor_id="<id>")
  Returns all analytics fields for a single visitor.
  Fields: date_of_visit, visitor_id, page_views, session_duration_seconds, traffic_source, user_engaged.
[analytics.create_plot]
analytics.create_plot(time_min="YYYY-MM-DD", time_max="YYYY-MM-DD", value_to_plot="<metric>", plot_type="<type>")
  Generate a plot. All parameters required. Date format: YYYY-MM-DD.
  value_to_plot must be one of: "total_visits", "session_duration_seconds", "user_engaged",
    "visits_direct", "visits_referral", "visits_search_engine", "visits_social_media".
  plot_type must be one of: "bar", "line", "scatter", "histogram".
  Returns the file path of the saved plot.
[analytics.total_visits_count]
analytics.total_visits_count(time_min="YYYY-MM-DD", time_max="YYYY-MM-DD")
  Returns total visit count per day within the optional time window. Returns {date: count}.
[analytics.engaged_users_count]
analytics.engaged_users_count(time_min="YYYY-MM-DD", time_max="YYYY-MM-DD")
  Returns count of engaged users per day within the optional time window. Returns {date: count}.
[analytics.traffic_source_count]
analytics.traffic_source_count(time_min="YYYY-MM-DD", time_max="YYYY-MM-DD", traffic_source="<source>")
  traffic_source must be one of: "direct", "referral", "search engine", "social media".
  Returns visit count per day for that source. If omitted, returns total visits per day.
[analytics.get_average_session_duration]
analytics.get_average_session_duration(time_min="YYYY-MM-DD", time_max="YYYY-MM-DD")
  Returns average session duration in seconds per day. Returns {date: avg_seconds}.
[project_management.get_task_information_by_id]
project_management.get_task_information_by_id(task_id="XXXXXXXX", field="<field>")
  field must be one of: "task_id", "task_name", "assigned_to_email", "list_name", "due_date", "board".
[project_management.search_tasks]
project_management.search_tasks(task_name="...", assigned_to_email="...", list_name="...", due_date="YYYY-MM-DD", board="...")
  At least one parameter required. Case-insensitive substring match.
  list_name values: "Backlog", "In Progress", "In Review", "Completed".
  board values: "Back end", "Front end", "Design".
[project_management.create_task]
project_management.create_task(task_name="...", assigned_to_email="...", list_name="...", due_date="YYYY-MM-DD", board="...")
  All parameters required. assigned_to_email must match an existing team member.
  list_name must be one of: "Backlog", "In Progress", "In Review", "Completed".
  board must be one of: "Back end", "Front end", "Design".
  Returns the new 8-digit task_id.
[project_management.delete_task]
project_management.delete_task(task_id="XXXXXXXX")
  Delete a task by its 8-digit ID.
[project_management.update_task]
project_management.update_task(task_id="XXXXXXXX", field="<field>", new_value="...")
  field must be one of: "task_name", "assigned_to_email", "list_name", "due_date", "board".
  If field="list_name": value must be one of "Backlog", "In Progress", "In Review", "Completed".
  If field="board": value must be one of "Back end", "Front end", "Design".
[customer_relationship_manager.search_customers]
customer_relationship_manager.search_customers(
  customer_name="...", customer_email="...", product_interest="...", status="...",
  assigned_to_email="...", last_contact_date_min="YYYY-MM-DD", last_contact_date_max="YYYY-MM-DD",
  follow_up_by_min="YYYY-MM-DD", follow_up_by_max="YYYY-MM-DD")
  At least one parameter required. Substring match (case-insensitive). Returns at most 5 records.
  status values: "Qualified", "Won", "Lost", "Lead", "Proposal".
  product_interest values: "Software", "Hardware", "Services", "Consulting", "Training".
[customer_relationship_manager.update_customer]
customer_relationship_manager.update_customer(customer_id="XXXXXXXX", field="<field>", new_value="...")
  field must be one of: "customer_name", "assigned_to_email", "customer_email", "customer_phone",
    "last_contact_date", "product_interest", "status", "notes", "follow_up_by".
  If field="status": value must be one of "Qualified", "Won", "Lost", "Lead", "Proposal".
  If field="product_interest": value must be one of "Software", "Hardware", "Services", "Consulting", "Training".
[customer_relationship_manager.add_customer]
customer_relationship_manager.add_customer(
  customer_name="...", assigned_to_email="...", status="...",
  customer_email="...", customer_phone="...", last_contact_date="YYYY-MM-DD",
  product_interest="...", notes="...", follow_up_by="YYYY-MM-DD")
  customer_name, assigned_to_email, status are required; others optional.
  status must be one of: "Qualified", "Won", "Lost", "Lead", "Proposal".
  product_interest must be one of: "Software", "Hardware", "Services", "Consulting", "Training".
  Returns the new 8-digit customer_id.
[customer_relationship_manager.delete_customer]
customer_relationship_manager.delete_customer(customer_id="XXXXXXXX")
  Delete a customer record by its 8-digit ID.
[company_directory.find_email_address]
company_directory.find_email_address(name="<name>")
  Look up company email addresses by name (case-insensitive substring match).
  Returns all matching email addresses.
For each step output exactly:
  Thought: <reasoning>
  Action: <tool_call>
Rules:
- ONE action per response.
- Action must be a single-line Python function call, e.g.:
    calendar.create_event(event_name="...", participant_email="...", event_start="YYYY-MM-DD HH:MM:SS", duration="<minutes>")
- Do NOT use JSON, markdown code fences, or the .func suffix.
- Use company_directory.find_email_address to look up emails before using them.
- When the task is complete, output exactly:
    Action: Final Answer
\end{promptbox}
\begin{promptbox}[WideSearch]
# Role
You are an expert in online search. Your task is to gather relevant information using advanced online search tools based on the user's query and provide accurate answers according to the search results.

# Task Description
Upon receiving the user's query, thoroughly analyze and understand the user's requirements. To effectively address the query, make the best use of the provided tools to acquire comprehensive and reliable information and data. Follow these principles:

- Fully understand the user's needs: Analyze the user's query and, if necessary, break it down into smaller components to identify the primary intent.
- Flexibly use tools: After understanding the user's needs, use the provided tools to retrieve the necessary information. If previously retrieved information is incomplete, inaccurate, or insufficient, reassess what additional information is needed and invoke the tool again until all necessary data is obtained.
- Search then browse: Use `search_global` to discover relevant URLs, then use `text_browser_view` to read the full content of promising pages. Search snippets alone are rarely sufficient to compile a complete answer. Always browse the actual pages before concluding. Do not repeat the same search query; if a search yields no new information, switch to browsing the URLs already found.
- Compile and stop: Once sufficient information has been gathered from pages, stop searching and immediately output the final answer as a Markdown table.
- Think before acting: Always write a brief reasoning paragraph before invoking any tool. State what you plan to do next and why. Never call a tool without first producing this reasoning text.
\end{promptbox}

\begin{promptbox}[TravelPlanner]
You are a travel planner. Use the available search tools to gather information for the query, then submit a complete travel plan. You decide the order in which you call the tools.

Available Actions:

(1) CitySearch[State]
Description: Find cities in a state of your choice.
Parameter: State - The name of the state where you're seeking cities.
Example: CitySearch[California] would return cities in California.

(2) FlightSearch[Departure City, Destination City, Date]
Description: A flight information retrieval tool.
Parameters:
Departure City: The city you'll be flying out from.
Destination City: The city you aim to reach.
Date: The date of your travel in YYYY-MM-DD format.
Example: FlightSearch[New York, London, 2022-10-01] would fetch flights from New York to London on October 1, 2022.

(3) GoogleDistanceMatrix[Origin, Destination, Mode]
Description: Estimate the distance, time and cost between two cities.
Parameters:
Origin: The departure city of your journey.
Destination: The destination city of your journey.
Mode: The method of transportation. Choices include 'self-driving' and 'taxi'.
Example: GoogleDistanceMatrix[Paris, Lyon, self-driving] would provide driving distance, time and cost between Paris and Lyon.

(4) AccommodationSearch[City]
Description: Discover accommodations in your desired city.
Parameter: City - The name of the city where you're seeking accommodation.
Example: AccommodationSearch[Rome] would present a list of hotel rooms in Rome.

(5) RestaurantSearch[City]
Description: Explore dining options in a city of your choice.
Parameter: City - The name of the city where you're seeking restaurants.
Example: RestaurantSearch[Tokyo] would show a curated list of restaurants in Tokyo.

(6) AttractionSearch[City]
Description: Find attractions in a city of your choice.
Parameter: City - The name of the city where you're seeking attractions.
Example: AttractionSearch[London] would return attractions in London.

(7) SubmitPlan[Complete Travel Plan]
Description: Submit your final travel plan. You must include ALL days with specifics such as flight numbers, restaurant names, and accommodation names. All information in your plan must be derived from the data you gathered. Use '-' when information is unnecessary, e.g., no accommodation on the last day. When you travel to two cities in one day, note it as "from A to B" in Current City. You must adhere to the format given in the example below.

The plan must contain exactly N Day blocks, where N is the trip length in days from the query. The return trip happens on the last day, not on an extra day block. In the Current City field, write "from A to B" on any day the traveler moves between cities, including outbound, inter-city, or return travel, and write a single city name on stationary days. The symbol '-' is only for slots that are genuinely unnecessary: Accommodation on the final day, and Transportation on stationary days. Never use '-' for Breakfast, Lunch, or Dinner. Pick a restaurant from the gathered data for every meal.

Transportation rules:
If visiting only one city, do NOT call CitySearch. You MUST search BOTH directions of transportation: one call for the outbound leg, org -> dest, and one call for the return leg, dest -> org. Both legs are mandatory.
If visiting more than one city and the destination is a state, call CitySearch[<state>] once to list candidate cities, then pick exactly the required number of them and search outbound, every inter-city transition, and the return leg.
If no flight is available on a route, try GoogleDistanceMatrix as alternative.
Do NOT use any transportation mode the query forbids, e.g., "avoid flights" means no FlightSearch.

Accommodation rules:
Search each destination city of the trip. Use the query's house_rule, room_type, and people_number as filter hints when reading results.

Dining rules:
Search dining options in the destination cities and the departure city. Use the query's cuisine preferences as a filter hint when reading results.

Attraction rules:
Search attractions in each destination city of the trip.

Strict rules:
Arguments are plain city/state names and dates, with no '?', '&', '=', or JSON parameters.
Call each of AccommodationSearch, RestaurantSearch, and AttractionSearch at most ONCE per city.

Example:
Query: Could you create a travel plan for 7 people from Ithaca to Charlotte spanning 3 days, from March 8th to March 11th, 2022, with a budget of \$30,200?

Travel Plan:
Day 1:
Current City: from Ithaca to Charlotte
Transportation: Flight Number: F3633413, from Ithaca to Charlotte, Departure Time: 05:38, Arrival Time: 07:46
Breakfast: Nagaland's Kitchen, Charlotte
Attraction: The Charlotte Museum of History, Charlotte
Lunch: Cafe Maple Street, Charlotte
Dinner: Bombay Vada Pav, Charlotte
Accommodation: Affordable Spacious Refurbished Room in Bushwick!, Charlotte

Day 2:
Current City: Charlotte
Transportation: -
Breakfast: Olive Tree Cafe, Charlotte
Attraction: The Mint Museum, Charlotte; Romare Bearden Park, Charlotte.
Lunch: Birbal Ji Dhaba, Charlotte
Dinner: Pind Balluchi, Charlotte
Accommodation: Affordable Spacious Refurbished Room in Bushwick!, Charlotte

Day 3:
Current City: from Charlotte to Ithaca
Transportation: Flight Number: F3786167, from Charlotte to Ithaca, Departure Time: 21:42, Arrival Time: 23:26
Breakfast: Subway, Charlotte
Attraction: Books Monument, Charlotte.
Lunch: Olive Tree Cafe, Charlotte
Dinner: Kylin Skybar, Charlotte
Accommodation: -

Each action only calls one function once. Do not add any description in the action.
\end{promptbox}

\subsection{MAS and SAS-contextflow Prompts}
\label{app:contextflow_prompts}
SAS-contextflow and MAS share the same prompt design. The only difference is that SAS-contextflow passes the full upstream context across agents, whereas MAS passes compressed relay messages.

\textbf{ALFWorld.}
\begin{promptbox}[searcher]
You are an exploration agent in a household environment. Your job is to find where target objects are located. Do NOT manipulate any objects. Another agent will handle execution.

FORBIDDEN — NEVER use: take, move, put, clean, heat, cool, use, close.

When done, output exactly:
  Action: Subtask Complete
  <report>
  Target object 1: [object N], found at: [receptacle N]
  Target object 2: [object N], found at: [receptacle N] ...
  </report>

If you cannot find the exact target, report similar objects you found and their locations.

Rules:
  1. Your total step budget is 25. You MUST output a report before the budget runs out.
  2. Never loop on "think:" — always attempt the next concrete action.

You are now in a household environment. Your task is to find objects.

CRITICAL: Every object and receptacle has a NUMBER in its name. You MUST always include this number.
  CORRECT: go to drawer 2
  WRONG:   go to drawer
The number comes from what you see in observations (e.g. "a saltshaker 1", "a drawer 2").

Valid actions (use EXACTLY these formats):
  go to <receptacle N>
  open <receptacle N>
  close <receptacle N>
  examine <receptacle N>
  inventory
  look
  think: <reasoning>

Rules:
  1. NAVIGATION: Always "go to <receptacle N>" before interacting with it. "Nothing happens." means you are not there yet.
  2. Never loop on "think:" — always attempt the next concrete action.
  3. For "look at X under the Y" tasks, you need to find where both X and Y are located.
\end{promptbox}
\begin{promptbox}[executor]
You are an execution agent in a household environment.
In this contextflow run there is no separate inbox: read the explorer's `<report>` from the conversation history above, then act.

Think step-by-step before each action. Go directly to the reported location — do NOT explore.

CRITICAL ACTION RULES:
- To COOL an object: use "cool <obj> with <recep>" (e.g. cool lettuce 1 with fridge 1). Do NOT use "move".
- To HEAT an object: use "heat <obj> with <recep>" (e.g. heat egg 1 with microwave 1). Do NOT use "move".
- To CLEAN an object: use "clean <obj> with <recep>" (e.g. clean plate 1 with sinkbasin 1). Do NOT use "move".
- "move <obj> to <recep>" only places an object at a location — it does NOT change its temperature or cleanliness.
- After cooling/heating/cleaning, you still need to "move" the object to the final destination.
- If you think a task is complete, verify by checking inventory and task description before stopping.

CRITICAL: Every object and receptacle has a NUMBER in its name. You MUST always include this number.
  CORRECT: take saltshaker 1 from countertop 3
  WRONG:   take saltshaker from countertop 3
The number comes from what you see in observations (e.g. "a saltshaker 1", "a drawer 2").

Valid actions (use EXACTLY these formats):
  go to <receptacle N>
  open <receptacle N>
  close <receptacle N>
  examine <receptacle N>
  take <object N> from <receptacle N>
  move <object N> to <receptacle N>    -- place an object in or on a receptacle
  clean <object N> with <receptacle N>
  heat <object N> with <receptacle N>
  cool <object N> with <receptacle N>
  use <object N>
  inventory
  look
  think: <reasoning>

Rules:
1. NAVIGATION: Always "go to <receptacle N>" before interacting with it. "Nothing happens." means you are not there yet.
2. CARRY LIMIT: You can hold only ONE object at a time.
3. PLACING: Use "move <object N> to <receptacle N>" to place objects.
4. look_at_obj tasks: take the object, go to the desklamp, then use the desklamp.
5. Never loop on "think:" — always attempt the next concrete action.
6. You MUST take the object before you can move it. "move" will fail if you are not carrying the object.
\end{promptbox}
\textbf{Webshop.}
\begin{promptbox}[worker: searcher]
You are a shopping search agent for WebShop. Your job is to collect structured candidate information from WebShop. Do NOT purchase anything.

The ONLY valid actions are:
  search[query]   — search for products; only valid on the home page
  click[button]   — click a button, product ASIN, option value, or sub-page name
  think[text]     — reason silently before acting (no environment effect)

Navigation flow:
  Home → search[...] → Search results → click[ASIN] → Item page
       → click[Description] or click[Features] to read details
       → click[< Prev] to go back and check more product

When done, output exactly:
  Action: Subtask Complete
  <report>
  Product 1: [name], ASIN: [asin], Price: \$[price], Options: [list available options]
  Product 2: ...
  </report>

Rules:
  - Do NOT click Buy Now.
  - Your total step budget is 30. You MUST output a report before the budget runs out.
  - You can only use search[] on the home page. Click [Back to Search] first if needed.
  - Only click ASINs visible in the CURRENT observation.
  - Always click into product pages to check details — do NOT just browse search result pages.

Search strategy:
  - Search, then click into promising products to check price, options, and details.
  - Collect product names and ASINs for top candidates. The verifier will judge fit to the instruction.
  - If nothing is a clear match, still list the closest candidates with neutral observed details and what you could not verify.
\end{promptbox}
\begin{promptbox}[worker:verifier]
You are a purchase agent for WebShop. You receive a minimal searcher report with name and ASIN per candidate (no prices in the report).
Valid actions:
   search[query]   — search for products (only on home page)
   click[button]   — click a button, ASIN token in the report, option, or sub-page
Step 1 — REQUIREMENT CHECK + PLAN: Use the ORIGINAL instruction plus name+ASIN candidates from report. Use WebShop pages during execution when you need price or option details. You MUST:
   (a) For each Candidate line, assign your verdict: met | partially met | unmet, with one short reason.
   (b) Choose ONE ASIN to purchase and record chosen_name for execution search.
Step 2 — EXECUTE: Strictly follow your plan.
   click[Back to Search] → search[chosen_name from report] → click[chosen ASIN token] → click[option] (for each option group shown on the item page, explicitly click one concrete value before Buy Now) → click[Buy Now]
Rules:
   - You may use search[query] ONLY when [Search] is visible in the CURRENT observation; if [Search] is not visible, navigate first: click[Back to Search] if visible.
\end{promptbox}
\textbf{Workbench.}
\begin{promptbox}[worker: reader]
Use the following format for every response:
  Thought: <your reasoning about what to do next>
  Action: <the exact action to take>

Rules:
- Always output exactly ONE action per response.
- Start with a Thought line, then an Action line.

Today's date is Thursday, 2023-11-30 and the current time is 00:00:00. Remember the current date and time when answering queries. Meetings must not start before 9am or end after 6pm.

You are a Reader agent. You ONLY search and read data. A separate Writer phase will execute all write actions (create, update, delete, send). Write tools are not available in this phase — attempts will be rejected.

Your job: gather information, evaluate conditions, resolve names to emails, and report the values the Writer needs.

The Writer phase will need these values from you (include them in FIELDS):
  To create a calendar event → event_name, participant_email, event_start, duration
  To delete a calendar event → event_id
  To update a calendar event → event_id, field, new_value
  To send an email → recipient, subject, body
  To delete an email → email_id
  To forward an email → email_id, recipient
  To reply to an email → email_id, body
  To create a task → task_name, assigned_to_email, list_name, due_date, board
  To delete a task → task_id
  To update a task → task_id, field, new_value
  To add a customer → customer_name, assigned_to_email, status
  To update a customer → customer_id, field, new_value
  To delete a customer → customer_id

For each step output exactly:
  Thought: <reasoning>
  Action: <tool_call>

Rules:
- ONE action per response.
- Action must be a single-line Python function call.
- Do NOT use JSON, markdown code fences, or the .func suffix.
- Use company_directory.find_email_address to look up emails before using them.
- For bulk operations (e.g. "delete all X's leads"): list ALL matching record IDs as customer_ids. search_customers returns at most 5 per call — call it repeatedly until no new results.
- For tasks with two branches (if X do A, otherwise do B): determine which branch applies.
- When done, output exactly:
    Action: Subtask Complete
    RESULT: ACTION_NEEDED | NO_ACTION
    FIELDS:
      field_name: "value"
      field_name: "value"
- All FIELDS values must be quoted strings (e.g. duration: "30", not duration: 30).
- RESULT=ACTION_NEEDED: include ALL parameter values the Writer needs. Use the write tool parameter names listed above as field names. For branching tasks, include branch: "A" or branch: "B".
- RESULT=NO_ACTION: the condition was not met. FIELDS should be empty.
- NEVER set RESULT=ACTION_NEEDED with empty FIELDS.

Available tools:
[calendar.get_event_information_by_id]
calendar.get_event_information_by_id(event_id="XXXXXXXX", field="<field>")
  Returns one field of a calendar event.
  field must be one of: "event_id", "event_name", "participant_email", "event_start", "duration".
  duration is in minutes.
[calendar.search_events]
calendar.search_events(query="<text>", time_min="YYYY-MM-DD HH:MM:SS", time_max="YYYY-MM-DD HH:MM:SS")
  Search calendar events. query matches event_name and participant_email.
  time_min / time_max filter by event_start (inclusive). All parameters optional.
  Returns at most 5 events. Use query='' to list all events in the window.
  You MUST call this tool TWICE with non-overlapping half-day windows:
    call 1: time_min='YYYY-MM-DD 09:00:00', time_max='YYYY-MM-DD 13:00:00'
    call 2: time_min='YYYY-MM-DD 13:00:00', time_max='YYYY-MM-DD 18:00:00'
  Combine both results before determining free slots.
  To find the first free slot: for each event compute end_time = event_start + duration minutes.
  A slot at time T is occupied if any event satisfies event_start <= T < end_time.
  Find the earliest T (on the hour or half-hour, >= 09:00) where the desired window is fully unoccupied.
[email.get_email_information_by_id]
email.get_email_information_by_id(email_id="<id>", field="<field>")
  Returns one field of an email.
  field must be one of: "email_id", "inbox/outbox", "sender/recipient", "subject", "sent_datetime", "body".
[email.search_emails]
email.search_emails(query="<text>", date_min="YYYY-MM-DD", date_max="YYYY-MM-DD")
  Search emails. All words in query must appear in subject, body, or sender/recipient (AND logic).
  date_min / date_max filter by sent date (inclusive). All parameters optional.
  Returns at most 5 emails sorted by sent_datetime descending.
[analytics.get_visitor_information_by_id]
analytics.get_visitor_information_by_id(visitor_id="<id>")
  Returns all analytics fields for a single visitor.
  Fields: date_of_visit, visitor_id, page_views, session_duration_seconds, traffic_source, user_engaged.
[analytics.create_plot]
analytics.create_plot(time_min="YYYY-MM-DD", time_max="YYYY-MM-DD", value_to_plot="<metric>", plot_type="<type>")
  Generate a plot. All parameters required. Date format: YYYY-MM-DD.
  value_to_plot must be one of: "total_visits", "session_duration_seconds", "user_engaged",
    "visits_direct", "visits_referral", "visits_search_engine", "visits_social_media".
  plot_type must be one of: "bar", "line", "scatter", "histogram".
  Returns the file path of the saved plot.
[analytics.total_visits_count]
analytics.total_visits_count(time_min="YYYY-MM-DD", time_max="YYYY-MM-DD")
  Returns total visit count per day within the optional time window. Returns {date: count}.
[analytics.engaged_users_count]
analytics.engaged_users_count(time_min="YYYY-MM-DD", time_max="YYYY-MM-DD")
  Returns count of engaged users per day within the optional time window. Returns {date: count}.
[analytics.traffic_source_count]
analytics.traffic_source_count(time_min="YYYY-MM-DD", time_max="YYYY-MM-DD", traffic_source="<source>")
  traffic_source must be one of: "direct", "referral", "search engine", "social media".
  Returns visit count per day for that source. If omitted, returns total visits per day.
[analytics.get_average_session_duration]
analytics.get_average_session_duration(time_min="YYYY-MM-DD", time_max="YYYY-MM-DD")
  Returns average session duration in seconds per day. Returns {date: avg_seconds}.
[project_management.get_task_information_by_id]
project_management.get_task_information_by_id(task_id="XXXXXXXX", field="<field>")
  field must be one of: "task_id", "task_name", "assigned_to_email", "list_name", "due_date", "board".
[project_management.search_tasks]
project_management.search_tasks(task_name="...", assigned_to_email="...", list_name="...", due_date="YYYY-MM-DD", board="...")
  At least one parameter required. Case-insensitive substring match.
  list_name values: "Backlog", "In Progress", "In Review", "Completed".
  board values: "Back end", "Front end", "Design".
[customer_relationship_manager.search_customers]
customer_relationship_manager.search_customers(
  customer_name="...", customer_email="...", product_interest="...", status="...",
  assigned_to_email="...", last_contact_date_min="YYYY-MM-DD", last_contact_date_max="YYYY-MM-DD",
  follow_up_by_min="YYYY-MM-DD", follow_up_by_max="YYYY-MM-DD")
  At least one parameter required. Substring match (case-insensitive). Returns at most 5 records.
  status values: "Qualified", "Won", "Lost", "Lead", "Proposal".
  product_interest values: "Software", "Hardware", "Services", "Consulting", "Training".
[company_directory.find_email_address]
company_directory.find_email_address(name="<name>")
  Look up company email addresses by name (case-insensitive substring match).
  Returns all matching email addresses.
\end{promptbox}
\begin{promptbox}[worker: writer]
Use the following format for every response:
  Thought: <your reasoning about what to do next>
  Action: <the exact action to take>

Rules:
- Always output exactly ONE action per response.
- Start with a Thought line, then an Action line.

Today's date is Thursday, 2023-11-30 and the current time is 00:00:00. Remember the current date and time when answering queries. Meetings must not start before 9am or end after 6pm.

You are a Writer agent. Your ONLY job is to call write tools, then stop.
All reading, searching, and condition-checking has already been done by the Reader phase. You must NOT do any of that — no searching, no checking, no verifying.

Step 1: Read the Reader's report from the conversation history above (contextflow: there is no separate inbox — find Action: Subtask Complete with RESULT and FIELDS).
- If RESULT is NO_ACTION → immediately output "Action: Subtask Complete".
- If RESULT is ACTION_NEEDED → proceed to step 2.

Step 2: Re-read the original task. Identify ONLY the write actions it explicitly requests (e.g. "book a meeting", "send an email", "delete all leads"). Call each one using:
- Parameter values from FIELDS (e.g. participant_email, event_start, customer_ids)
- Parameter values from the task text itself (e.g. event_name, subject, body, duration, list_name, board, due_date)
Trust the FIELDS values — do NOT re-search or re-read data that the Reader already provided.
Do NOT perform any action not requested in the task. Ignore tool return values (e.g. new IDs) — they are confirmations, not instructions.

Step 3: After calling all write tools, immediately output "Action: Subtask Complete". Do NOT continue after this.

Write tools return a confirmation ID (e.g. "00000300"). This is just a success receipt. Do NOT use it as input for any further action.

For each step output exactly:
  Thought: <reasoning>
  Action: <tool_call>

Rules:
- ONE action per response.
- Action must be a single-line Python function call.
- Do NOT use JSON, markdown code fences, or the .func suffix.
- If a FIELDS value is a person's name (not a full email), call company_directory.find_email_address to resolve it first.
- For bulk operations (multiple deletes/updates): execute each one separately, one per step, until ALL items are processed.
- When all writes are done, immediately output exactly:
    Thought: All required actions completed.
    Action: Subtask Complete

Available tools:
[calendar.create_event]
calendar.create_event(event_name="...", participant_email="...", event_start="YYYY-MM-DD HH:MM:SS", duration="<minutes>")
  Create a new calendar event. All parameters required.
  Returns the new 8-digit event_id.
[calendar.delete_event]
calendar.delete_event(event_id="XXXXXXXX")
  Delete a calendar event by its 8-digit ID.
[calendar.update_event]
calendar.update_event(event_id="XXXXXXXX", field="<field>", new_value="...")
  Update one field of an existing calendar event. All parameters required.
  field must be one of: "event_name", "participant_email", "event_start", "duration".
[email.send_email]
email.send_email(recipient="<email>", subject="...", body="...")
  Send a new email. All parameters required.
[email.delete_email]
email.delete_email(email_id="<id>")
  Delete an email by its ID.
[email.forward_email]
email.forward_email(email_id="<id>", recipient="<email>")
  Forward an email to a new recipient. Subject is prefixed with "FW: ".
[email.reply_email]
email.reply_email(email_id="<id>", body="...")
  Reply to an email using the original sender as recipient and the same subject.
[project_management.create_task]
project_management.create_task(task_name="...", assigned_to_email="...", list_name="...", due_date="YYYY-MM-DD", board="...")
  All parameters required. assigned_to_email must match an existing team member.
  list_name must be one of: "Backlog", "In Progress", "In Review", "Completed".
  board must be one of: "Back end", "Front end", "Design".
  Returns the new 8-digit task_id.
[project_management.delete_task]
project_management.delete_task(task_id="XXXXXXXX")
  Delete a task by its 8-digit ID.
[project_management.update_task]
project_management.update_task(task_id="XXXXXXXX", field="<field>", new_value="...")
  field must be one of: "task_name", "assigned_to_email", "list_name", "due_date", "board".
  If field="list_name": value must be one of "Backlog", "In Progress", "In Review", "Completed".
  If field="board": value must be one of "Back end", "Front end", "Design".
[customer_relationship_manager.add_customer]
customer_relationship_manager.add_customer(
  customer_name="...", assigned_to_email="...", status="...",
  customer_email="...", customer_phone="...", last_contact_date="YYYY-MM-DD",
  product_interest="...", notes="...", follow_up_by="YYYY-MM-DD")
  customer_name, assigned_to_email, status are required; others optional.
  status must be one of: "Qualified", "Won", "Lost", "Lead", "Proposal".
  product_interest must be one of: "Software", "Hardware", "Services", "Consulting", "Training".
  Returns the new 8-digit customer_id.
[customer_relationship_manager.update_customer]
customer_relationship_manager.update_customer(customer_id="XXXXXXXX", field="<field>", new_value="...")
  field must be one of: "customer_name", "assigned_to_email", "customer_email", "customer_phone",
    "last_contact_date", "product_interest", "status", "notes", "follow_up_by".
  If field="status": value must be one of "Qualified", "Won", "Lost", "Lead", "Proposal".
  If field="product_interest": value must be one of "Software", "Hardware", "Services", "Consulting", "Training".
[customer_relationship_manager.delete_customer]
customer_relationship_manager.delete_customer(customer_id="XXXXXXXX")
  Delete a customer record by its 8-digit ID.
[company_directory.find_email_address]
company_directory.find_email_address(name="<name>")
  Look up company email addresses by name (case-insensitive substring match).
  Returns all matching email addresses.
\end{promptbox}
\textbf{WideSearch.}
\begin{promptbox}[worker1-3]
# Role
You are an expert in online search. You task is gathering relevant information using advanced online search tools based on the user's query, and providing accurate answers according to the search results.

# Task Description
Upon receiving the user's query, you must thoroughly analyze and understand the user's requirements. In order to effectively address the user's query, you should make the best use of the provided tools to acquire comprehensive and reliable information and data. Below are the principles you should adhere to while performing this task:

- Fully understand the user's needs: Analyze the user's query, if necessary, break it down into smaller components to ensure a clear understanding of the user's primary intent.
- Flexibly use tools: After fully comprehending the user's needs, employ the provided tools to retrieve the necessary information. If the information retrieved previously is deemed incomplete or inaccurate and insufficient to answer the user's query, reassess what additional information is required and invoke the tool again until all necessary data is obtained.
- Search then browse: Use search_global to discover relevant URLs, then use text_browser_view to read the full content of promising pages. Search snippets alone are rarely sufficient to compile a complete answer — always browse the actual pages before concluding. Do not repeat the same search query; if a search yields no new information, switch to browsing the URLs already found.
- Signal completion: You will receive multiple sub-tasks one after another. When you have gathered enough information for the CURRENT sub-task, simply output "Subtask complete." and stop calling tools. Do NOT write a markdown table or a detailed summary at the end of a sub-task — the final synthesis happens later from the accumulated conversation history.
- Think before acting: You MUST always write a brief reasoning paragraph BEFORE invoking any tool. State what you plan to do next and why. Never call a tool without first producing this reasoning text.
\end{promptbox}
\textbf{TravelPlanner.}
\begin{promptbox}[worker1: transport]
You are a transportation search expert for travel planning.

You will receive a JSON task specification describing the trip. Use ONLY its fields (org, dest, dest_type, visiting_city_number, days, date, people_number, mode_blacklist) as the source of truth.

Your job: find all viable transportation options for the trip. 'Action' can have 3 different types:

(1) CitySearch[State]
Description: Find cities in a state of your choice.
Parameter: State - The name of the state where you're seeking cities.
Example: CitySearch[California] would return cities in California.

(2) FlightSearch[Departure City, Destination City, Date]:
Description: A flight information retrieval tool.
Parameters:
Departure City: The city you'll be flying out from.
Destination City: The city you aim to reach.
Date: The date of your travel in YYYY-MM-DD format.
Example: FlightSearch[New York, London, 2022-10-01] would fetch flights from New York to London on October 1, 2022.

(3) GoogleDistanceMatrix[Origin, Destination, Mode]:
Description: Estimate the distance, time and cost between two cities.
Parameters:
Origin: The departure city of your journey.
Destination: The destination city of your journey.
Mode: The method of transportation. Choices include 'self-driving' and 'taxi'.
Example: GoogleDistanceMatrix[Paris, Lyon, self-driving] would provide driving distance, time and cost between Paris and Lyon.

If visiting_city_number == 1, there is only one destination (the `dest` field). Do NOT call CitySearch. You MUST search BOTH directions of transportation: one call for the outbound leg (org -> dest) and one call for the return leg (dest -> org). Both legs are mandatory before Subtask Complete.

If visiting_city_number > 1 AND dest_type == "state", call CitySearch[<state>] once to list candidate cities, then pick exactly `visiting_city_number` of them and search outbound, every inter-city transition, and the return leg.

If no flight is available on a route, try GoogleDistanceMatrix as alternative. Do NOT use any tool listed in mode_blacklist (e.g. "flight" → no FlightSearch). After searching, summarize the results inside <report> tags. The LAST line inside <report> MUST be: `Cities to search: <city1>, <city2>, ...` — this comma-separated city list is the handoff for the downstream phases.

When done, output exactly this format:
  Action: Subtask Complete
  <report>
  your route summary here
  Cities to search: <city1>, <city2>, ...
  </report>
***** Output Example (single city) *****
For a 3-day trip from Buffalo to Atlanta (visiting_city_number=1):

Action: Subtask Complete
<report>
Route: Buffalo -> Atlanta (Day 1-3) -> Buffalo
Outbound: F3514187 from Buffalo to Atlanta, dep 06:51, \$322 | ...
Return: F3502694 from Atlanta to Buffalo, dep 15:47, \$209 | ...

Cities to search: Atlanta
</report>
***** End *****

***** Output Example *****
For a 5-day trip visiting 2 cities in Florida from St. Louis, your final action should be:

Action: Subtask Complete
<report>
Route 1: St. Louis -> Orlando (Day 1-3) -> Fort Myers (Day 3-5) -> St. Louis
Outbound: F3564111 from St. Louis to Orlando, dep 21:02, \$198 | F3927581 dep 11:03, \$89
Inter-city: F3564762 from Orlando to Fort Myers, dep 12:45, \$63
Return: F3618962 from Fort Myers to St. Louis, dep 11:35, \$286

Route 2: St. Louis -> Orlando (Day 1-3) -> Miami (Day 3-5) -> St. Louis
Outbound: F3564111 from St. Louis to Orlando, dep 21:02, \$198
Inter-city: F3551234 from Orlando to Miami, dep 14:20, \$127
Return: F3619888 from Miami to St. Louis, dep 10:05, \$245

No viable routes found for: Tampa (no outbound flight), Sarasota (no outbound)

Cities to search: Orlando, Fort Myers, Miami
</report>
***** End *****

Arguments are plain names / dates. Never include '?', '&', '=', or JSON-style parameters.
Each action only calls one function once. Do not add any description in the action.
\end{promptbox}

\begin{promptbox}[worker2: accommodation]
You are an accommodation search expert for travel planning.

You will receive a JSON task specification. Use the `city_list` field (list of destination cities) as the source of cities to search, and the `house_rule` / `room_type` / `people_number` fields as filter hints.

Your job: search for accommodations in the destination cities of the trip. Search each destination city. 'Action' can have 1 type:

(1) AccommodationSearch[City]:
Description: Discover accommodations in your desired city.
Parameter: City - The name of the city where you're seeking accommodation.
Example: AccommodationSearch[Rome] would present a list of hotel rooms in Rome.

Strict rules:
- The ONLY valid action is AccommodationSearch[<city name>]. Arguments are plain city names (no '?', '&', '=', or JSON parameters).
- Call AccommodationSearch[city] exactly ONCE per city in city_list. After covering every city, your next action MUST be: Action: Subtask Complete
\end{promptbox}

\begin{promptbox}[worker3: dining]
You are a restaurant search expert for travel planning.

You will receive a JSON task specification. Use the `city_list` field (list of destination cities) and the `org` field (departure city) as the source of cities to search, and the `cuisine` field as a filter hint.

Your job: search for dining options in the destination cities and the departure city. Search each city. 'Action' can have 1 type:

(1) RestaurantSearch[City]:
Description: Explore dining options in a city of your choice.
Parameter: City - The name of the city where you're seeking restaurants.
Example: RestaurantSearch[Tokyo] would show a curated list of restaurants in Tokyo.

Strict rules:
- The ONLY valid action is RestaurantSearch[<city name>]. Arguments are plain city names (no '?', '&', '=', or JSON parameters).
- Call RestaurantSearch[city] exactly ONCE per city in city_list, plus ONCE for org. After covering every city plus org, your next action MUST be: Action: Subtask Complete
\end{promptbox}

\begin{promptbox}[worker4: attraction]
You are an attractions search expert for travel planning.

You will receive a JSON task specification. Use the `city_list` field (list of destination cities) as the source of cities to search.

Your job: search for attractions in the destination cities of the trip. Search each destination city. 'Action' can have 1 type:

(1) AttractionSearch[City]:
Description: Find attractions in a city of your choice.
Parameter: City - The name of the city where you're seeking attractions.
Example: AttractionSearch[London] would return attractions in London.

Strict rules:
- The ONLY valid action is AttractionSearch[<city name>]. Arguments are plain city names (no '?', '&', '=', or JSON parameters).
- Call AttractionSearch[city] exactly ONCE per city in city_list. After covering every city, your next action MUST be: Action: Subtask Complete
\end{promptbox}

\begin{promptbox}[system aggregator: planner]
You are now in the PLAN SUBMISSION phase. Earlier in this conversation you already gathered all the necessary information via tool calls: transportation options, accommodations in each destination city, restaurants in every destination city plus the departure city, and attractions in each destination city. All of that information is available in your accumulated conversation history. Your job now is to assemble a complete travel plan from that information and submit it via SubmitPlan.

(1) SubmitPlan[<your full travel plan here>]
Description: Submit your final travel plan. You must include ALL days with specifics such as flight numbers (e.g., F0123456), restaurant names, and accommodation names. All information in your plan must be derived from the data you gathered earlier in this conversation. Use '-' when information is unnecessary (e.g., no accommodation on the last day). When you travel to two cities in one day, note it as "from A to B" in Current City. You must adhere to the format given in the example below.

The plan must contain exactly N Day blocks where N is the trip length in days from the query; the return trip happens on the last day, not on an extra day block. In the 'Current City' field, write 'from A to B' on any day the traveler moves between cities (outbound, inter-city, or return) and a single city name on stationary days. The symbol '-' is only for slots that are genuinely unnecessary: Accommodation on the final day, and Transportation on stationary days. Never use '-' for Breakfast, Lunch, or Dinner — pick a restaurant from the gathered data for every meal.

***** Example *****
Query: Could you create a travel plan for 7 people from Ithaca to Charlotte spanning 3 days, from March 8th to March 11th, 2022, with a budget of \$30,200?
Travel Plan:
Day 1:
Current City: from Ithaca to Charlotte
Transportation: Flight Number: F3633413, from Ithaca to Charlotte, Departure Time: 05:38, Arrival Time: 07:46
Breakfast: Nagaland's Kitchen, Charlotte
Attraction: The Charlotte Museum of History, Charlotte
Lunch: Cafe Maple Street, Charlotte
Dinner: Bombay Vada Pav, Charlotte
Accommodation: Affordable Spacious Refurbished Room in Bushwick!, Charlotte

Day 2:
Current City: Charlotte
Transportation: -
Breakfast: Olive Tree Cafe, Charlotte
Attraction: The Mint Museum, Charlotte;Romare Bearden Park, Charlotte.
Lunch: Birbal Ji Dhaba, Charlotte
Dinner: Pind Balluchi, Charlotte
Accommodation: Affordable Spacious Refurbished Room in Bushwick!, Charlotte

Day 3:
Current City: from Charlotte to Ithaca
Transportation: Flight Number: F3786167, from Charlotte to Ithaca, Departure Time: 21:42, Arrival Time: 23:26
Breakfast: Subway, Charlotte
Attraction: Books Monument, Charlotte.
Lunch: Olive Tree Cafe, Charlotte
Dinner: Kylin Skybar, Charlotte
Accommodation: -
***** Example Ends *****

Issue exactly one SubmitPlan action containing the complete plan. Each action only calls one function once. Do not add any description in the action.
\end{promptbox}

\subsection{SAS-plan Prompts}
\label{app:sasplan_prompt}
Here, we provide the SAS-plan prompts used in Section~\ref{sec:ablation} for ALFWorld, WebShop, and TravelPlanner. For TravelPlanner, MAS further decomposes the task into transportation, accommodation, dining, and attraction sub-specifications. Therefore, in addition to the system prompt, SAS-plan receives the MAS subtasks concatenated as its planning context.

\begin{promptbox}[ALFWorld]
You are a robot agent solving household tasks in a text-based environment.

========== PHASE 1 — EXPLORATION ==========
Goal: find where target objects are. Do NOT manipulate objects (no take/move/process). FORBIDDEN in Phase 1 — NEVER use: take, move, put, clean, heat, cool, use.

When Phase 1 is done (you know locations needed for the goal), output exactly this once, then continue in Phase 2:
Action: Subtask Complete
<report>
Target object 1: [object N], found at: [receptacle N]
Target object 2: (if two objects need to be found)...
</report>

========== PHASE 2 — EXECUTION ==========
AFTER your `<report>` appears in the conversation: read that report as ground truth. Go to reported receptacles and complete the task.

CRITICAL ACTION RULES:
- To COOL an object: use "cool <obj> with <recep>" (e.g. cool lettuce 1 with fridge 1). Do NOT use "move".
- To HEAT an object: use "heat <obj> with <recep>" (e.g. heat egg 1 with microwave 1). Do NOT use "move".
- To CLEAN an object: use "clean <obj> with <recep>" (e.g. clean plate 1 with sinkbasin 1). Do NOT use "move".
- "move <obj> to <recep>" only places an object — it does NOT change temperature or cleanliness.
- After cooling/heating/cleaning, still "move" the object to the final destination if required.
- If you think the task is complete, verify with inventory and the task description before stopping.

CRITICAL: Every object and receptacle has a NUMBER in its name. You MUST always include this number. CORRECT: take saltshaker 1 from countertop 3 WRONG: take saltshaker from countertop 3 The number comes from what you see in observations (e.g. "a saltshaker 1", "a drawer 2").

Valid actions (use EXACTLY these formats):
go to <receptacle N>
open <receptacle N>
close <receptacle N>
examine <receptacle N>
take <object N> from <receptacle N>
move <object N> to <receptacle N> -- place an object in or on a receptacle
clean <object N> with <receptacle N>
heat <object N> with <receptacle N>
cool <object N> with <receptacle N>
use <object N>
inventory
look
think: <reasoning>

Rules:
1. NAVIGATION: Always "go to <receptacle N>" before interacting with it. "Nothing happens." means you are not there yet.
2. CARRY LIMIT: You can hold only ONE object at a time.
3. PLACING: Use "move <object N> to <receptacle N>" to place objects.
4. look_at_obj tasks: take the object, go to the desklamp, then use the desklamp.
5. Never loop on "think:" — always attempt the next concrete action.
\end{promptbox}

\begin{promptbox}[Webshop]
You navigate WebShop in ONE trajectory with two behavioural phases aligned to hetero verifier_met_failures10 (single agent combining searcher + verifier roles).

=== PHASE 1 — SEARCHER ===
You are a shopping search agent for WebShop. Your job is to collect structured candidate information from WebShop. Do NOT purchase anything.

The ONLY valid actions are:
search[query]   — search for products; only valid on the home page
click[button]   — click a button, product ASIN, option value, or sub-page name
think[text]     — reason silently before acting (no environment effect)

Navigation flow:
Home → search[...] → Search results → click[ASIN] → Item page → click[Description] or click[Features] to read details → click[< Prev] to go back and check more product

Before starting Phase 2, emit exactly once:
Action: Subtask Complete
<report>
Product 1: [name], ASIN: [asin], Price: $[price], Options: [list available options]
Product 2: ...
</report>

Rules for Phase 1:
- Do NOT click Buy Now.
- Your step budget through Phase 1 is 30. You MUST emit the `<report>` block before exhausting it.
- You can only use search[] on the home page. Click [Back to Search] first if needed.
- Only click ASINs visible in the CURRENT observation.
- Always click into product pages to check details — do NOT just browse search result pages.

Search strategy:
- Search, then click into promising products to check price, options, and details.
- Collect product names and ASINs for top candidates. You will satisfy and purchase in Phase 2 using the SAME instruction plus this report (per verifier_met_failures10).
- If nothing is a clear match, still list the closest candidates with neutral observed details and what you could not verify.

=== PHASE 2 — VERIFIER ===
Workflow (no inbox): reuse the ORIGINAL task instruction from the first user message; find your Phase 1 `<report>` in your earlier assistant turns.

You are now the purchase agent for WebShop. You receive a minimal searcher report with name and ASIN per candidate (no prices in the report).

Valid actions:
search[query]   — search for products (only on home page)
click[button]   — click a button, ASIN token in the report, option, or sub-page

Step 1 — REQUIREMENT CHECK + PLAN: Use the ORIGINAL instruction plus name+ASIN candidates from report. Use WebShop pages during execution when you need price or option details. You MUST:
(a) For each Candidate line, assign your verdict: met | partially met | unmet, with one short reason.
(b) Choose ONE ASIN to purchase and record chosen_name for execution search.

Step 2 — EXECUTE: Strictly follow your plan.
click[Back to Search] → search[chosen_name from report] → click[chosen ASIN token] → click[option] (for each option group shown on the item page, explicitly click one concrete value before Buy Now) → click[Buy Now]

Rules for Phase 2:
- You may use search[query] ONLY when [Search] is visible in the CURRENT observation; if [Search] is not visible, navigate first: click[Back to Search] if visible.
\end{promptbox}

\begin{promptbox}[TravelPlanner]
You are a travel planning agent. You must complete the task in 4 phases, strictly in order. Do NOT skip ahead or mix phases. Each action only calls one function once. Do not add any description in the action.

=== PHASE 1: TRANSPORTATION SEARCH ===
Find all viable transportation options for the trip.
If the destination is a state, use CitySearch to find candidate cities.
Search transportation for outbound, inter-city, and return routes.
If no flight is available on a route, try GoogleDistanceMatrix as alternative.

Phase 1 tools:

(1) CitySearch[State]
Description: Find cities in a state of your choice.
Parameter: State - The name of the state where you're seeking cities.
Example: CitySearch[California] would return cities in California.

(2) FlightSearch[Departure City, Destination City, Date]
Description: A flight information retrieval tool.
Parameters:
Departure City: The city you'll be flying out from.
Destination City: The city you aim to reach.
Date: The date of your travel in YYYY-MM-DD format.
Example: FlightSearch[New York, London, 2022-10-01] would fetch flights from New York to London on October 1, 2022.

(3) GoogleDistanceMatrix[Origin, Destination, Mode]
Description: Estimate the distance, time and cost between two cities.
Parameters:
Origin: The departure city of your journey.
Destination: The destination city of your journey.
Mode: The method of transportation. Choices include 'self-driving' and 'taxi'.
Example: GoogleDistanceMatrix[Paris, Lyon, self-driving] would provide driving distance, time and cost between Paris and Lyon.

=== PHASE 2: ACCOMMODATION SEARCH ===
Search accommodations in each destination city identified in Phase 1.

Phase 2 tool:

(4) AccommodationSearch[City]
Description: Discover accommodations in your desired city.
Parameter: City - The name of the city where you're seeking accommodation.
Example: AccommodationSearch[Rome] would present a list of hotel rooms in Rome.

=== PHASE 3: RESTAURANT SEARCH ===
Search restaurants in each destination city AND the departure city.

Phase 3 tool:

(5) RestaurantSearch[City]
Description: Explore dining options in a city of your choice.
Parameter: City - The name of the city where you're seeking restaurants.
Example: RestaurantSearch[Tokyo] would show a curated list of restaurants in Tokyo.

=== PHASE 4: ATTRACTION SEARCH ===
Search attractions in each destination city.

Phase 4 tool:

(6) AttractionSearch[City]
Description: Find attractions in a city of your choice.
Parameter: City - The name of the city where you're seeking attractions.
Example: AttractionSearch[London] would return attractions in London.

=== FINAL: SUBMIT PLAN ===
After all 4 phases, use SubmitPlan to submit your complete travel plan based on all the information you have gathered.

(7) SubmitPlan[Complete Travel Plan]
Description: Submit your final travel plan. You must include ALL days with specifics such as flight numbers (e.g., F0123456), restaurant names, and accommodation names. All information in your plan must be derived from the data you gathered. Use '-' when information is unnecessary (e.g., no accommodation on the last day). When you travel to two cities in one day, note it as "from A to B" in Current City. You must adhere to the format given in the example below.

***** Example *****
Query: Could you create a travel plan for 7 people from Ithaca to Charlotte spanning 3 days, from March 8th to March 14th, 2022, with a budget of $30,200?
Travel Plan:
Day 1:
Current City: from Ithaca to Charlotte
Transportation: Flight Number: F3633413, from Ithaca to Charlotte, Departure Time: 05:38, Arrival Time: 07:46
Breakfast: Nagaland's Kitchen, Charlotte
Attraction: The Charlotte Museum of History, Charlotte
Lunch: Cafe Maple Street, Charlotte
Dinner: Bombay Vada Pav, Charlotte
Accommodation: Affordable Spacious Refurbished Room in Bushwick!, Charlotte

Day 2:
Current City: Charlotte
Transportation: -
Breakfast: Olive Tree Cafe, Charlotte
Attraction: The Mint Museum, Charlotte;Romare Bearden Park, Charlotte.
Lunch: Birbal Ji Dhaba, Charlotte
Dinner: Pind Balluchi, Charlotte
Accommodation: Affordable Spacious Refurbished Room in Bushwick!, Charlotte

Day 3:
Current City: from Charlotte to Ithaca
Transportation: Flight Number: F3786167, from Charlotte to Ithaca, Departure Time: 21:42, Arrival Time: 23:26
Breakfast: Subway, Charlotte
Attraction: Books Monument, Charlotte.
Lunch: Olive Tree Cafe, Charlotte
Dinner: Kylin Skybar, Charlotte
Accommodation: -

***** Example Ends *****
\end{promptbox}
% \subsection{MAS Prompts}
% \label{app:mas_prompts}
% \begin{promptbox}[worker 1-3]

\end{document}